\documentclass[opre,nonblindrev]{informs3} 

\usepackage{graphicx}
\usepackage{dsfont}
\usepackage{stmaryrd}
\usepackage[english]{babel}
\usepackage{algorithm}
\usepackage[noend]{algpseudocode}
\usepackage[utf8]{inputenc} 
\usepackage[T1]{fontenc}    
\usepackage{hyperref}       
\usepackage{url}            
\usepackage{booktabs}       
\usepackage{amsfonts}       
\usepackage{nicefrac}       
\usepackage{microtype}      
\usepackage{booktabs}
\usepackage{color}
\usepackage{bm}
\usepackage{bbm}
\usepackage{amsmath}
\usepackage{amssymb}
\usepackage{caption}
\usepackage{xcolor} 
\usepackage{subfigure}
\usepackage{pdfpages}
\usepackage[numbers, compress, sort]{natbib}
\usepackage{wrapfig}
\usepackage{comment}
\usepackage{multicol}
\usepackage{multirow, makecell}
\usepackage{hyperref}
\usepackage{mathtools}
\usepackage{url}
\usepackage{xcolor}
\usepackage{fontawesome5}

\DoubleSpacedXI 

\usepackage{endnotes}
\let\footnote=\endnote
\let\enotesize=\normalsize
\def\notesname{Endnotes}%
\def\makeenmark{$^{\theenmark}$}
\def\enoteformat{\rightskip0pt\leftskip0pt\parindent=1.75em
  \leavevmode\llap{\theenmark.\enskip}}

\usepackage{natbib}
 \bibpunct[, ]{(}{)}{,}{a}{}{,}%
 \def\bibfont{\small}%
\TheoremsNumberedThrough     
\ECRepeatTheorems

\EquationsNumberedThrough    

\begin{document}


\RUNAUTHOR{Anh-Nguyen and Huchette}

\RUNTITLE{Neural Network Verification as Piecewise Linear Optimization}

\TITLE{Neural Network Verification as Piecewise Linear Optimization:\\
Formulations for the Composition of Staircase Functions}

\ARTICLEAUTHORS{%
\AUTHOR{Tu Anh-Nguyen}
\AFF{Computational Applied Mathematics and Operations Research, Rice University, Houston, TX 77005, \EMAIL{tan5@rice.edu}} 
\AUTHOR{Joey Huchette}
\AFF{Google Research, \EMAIL{jhuchette@google.com}}
} 

\ABSTRACT{%
We present a technique for neural network verification using mixed-integer programming (MIP) formulations.
We derive a \emph{strong formulation} for each neuron in a network using piecewise linear activation functions.
Additionally, as in general, these formulations may require an exponential number of inequalities, we also derive a separation procedure that runs in super-linear time in the input dimension.
We first introduce and develop our technique on the class of \emph{staircase} functions, which generalizes the ReLU, binarized, and quantized activation functions.
We then use results for staircase activation functions to obtain a separation method for general piecewise linear activation functions.
Empirically, using our strong formulation and separation technique, we can reduce the computational time in exact verification settings based on MIP and improve the false negative rate for inexact verifiers relying on the relaxation of the MIP formulation.
}%


\HISTORY{This paper was first submitted on November 1, 2022}

\maketitle

%


\section{Introduction}

Neural networks, especially convolutional \citep{lecun1995convolutional} and deep architectures \citep{schmidhuber2015deep}, have solved many challenging problems in machine learning.
However, deep neural networks have recently been shown to be vulnerable to adversarial attacks \citep{xu2020adversarial}: some small, carefully chosen perturbations to the inputs can significantly affect the outputs of a trained deep neural network \citep{szegedy2014intriguing, goodfellow2014explaining}.
Researchers have demonstrated the adversarial vulnerabilities in many applications such as computer vision, self-driving cars~\citep{eykholt2018robust}, malware detection systems~\citep{grosse2016adversarial}, face recognition systems~\citep{sharif2016accessorize}, and natural language processing systems~\citep{jia2017adversarial}.
Such vulnerabilities in neural networks, combined with their pervasiveness in real-world applications, can pose a critical security issue.
Therefore, evaluating or verifying the robustness of neural network systems has become an essential step in the machine learning workflow \citep{carlini2017towards, papernot2016limitations}.
Verification techniques can be divided into inexact and exact methods; in both cases, verification is conducted by solving one or more mathematical optimization problems.
Inexact approaches, where we compromise accuracy for running time, are often based on linear programming (LP), whereas exact verifiers often utilize mixed-integer programming (MIP), limiting their scalability.
Commonly, the optimization problems are constructed by modeling each neuron separately, and then composing these neuron models together \citep{anderson2020strong, bunel2020branch, anderson2020tightened, dvijotham2018dual, tjandraatmadja2020convex, han2021single}.
The strength of the models for each neuron can greatly influence the performance of a verifier.
Strong formulations can reduce the solving time for exact MIP verifiers and improve the false negative rate of inexact LP verifiers.
Unfortunately, strong formulations tend to have an excessive number of constraints, which requires the verification method to solve large optimization problems.
A technique called \emph{cut generation} is often employed in dealing with such scalability issues.  
In this scheme, we start the problem with a few constraints, and others are added iteratively.
At each iteration, we need to answer whether a solution is optimal, and if not, we need to find a constraint that violates it.
\citep{anderson2020strong} proposed a separation method for a strong formulation of neurons with max-pooling or ReLU as their activation functions.
However, extending this approach to general discontinuous piecewise linear activation functions is not trivial. 
\\

\noindent
An important class of neural networks, binarized neural networks (BNN) \citep{hubara2016binarized} are neural networks where the networks' parameters and activation values are constrained to be $-1$ or $+1$ (i.e., values are represented by a single bit).
Due to their compactness and comparable performance with standard architectures \citep{rastegari2016xnor, zhou2016dorefa, tang2017train, lin2017towards}, this class of neural networks is frequently deployed on edge devices.
Quantized neural networks generalize BNN~\citep{hubara2017quantized} by using a small number of bits to represent values.
Of note, binarized or quantized activation functions, unlike max-pooling and ReLU, are generally non-convex and discontinuous.
Moreover, existing separation techniques do not apply to neural network verification, whose activation functions have more than three pieces. \\

\noindent
We summarize our contributions as follows.
\begin{enumerate}
    \item \emph{Efficient separation for ``simple'' discontinuous activation functions}. We first derive a separation procedure for a strong formulation of neurons with staircase activation functions that runs in $O(n \text{log}(n) + \max \{k,n\})$ time complexity, where $n$ is the dimension of the input to a neuron and $k$ is the number of pieces of a \emph{staircase function}.
    For an illustration of a staircase function, see Figure \ref{fig:main}.
    Notably, this result is an extension of the ReLU activation function separation method presented by Anderson et al.~\citep{anderson2020strong}. 
    Moreover, since quantized activation functions are special cases of staircase functions, our technique can be applied for binarized or quantized neural network verification.
    \item \emph{Separation for any piecewise linear activation function}. Using the separation for staircase functions as a building block, we derive a separation method for verifying any neural nets with piecewise linear activation functions. For neural network architectures with piecewise linear activation function, we refer to \citep{agostinelli2014learning, zeng2010multistability}.
    Our result gives a strong MIP formulation for any piecewise linear function $g: \mathds{R}^n \rightarrow \mathds{R}$ whose pieces are defined on a box domain and bounded between two parallel hyperplanes, see Figure \ref{fig:pwl-r2} for an illustration.\\
\end{enumerate} 

\noindent
The rest of the paper is organized as follows.
In Section \ref{problem formulation}, we will formally write the neural network verification task as an optimization problem and discuss how we use mixed-integer programming formulations to model neurons with staircase activation functions.
Section \ref{extreme} studies the properties of such a formulation, which will be exploited in Section \ref{sep procedure} to derive a fast separation procedure.
Section \ref{General piecewise linear function} describes how we can use the aforementioned results in Section \ref{extreme} and Section \ref{sep procedure} as building blocks for separating general piecewise linear activation functions.
Finally, Section \ref{sec:experiments} shows how our formulation and the cutting plane procedure perform in verifying quantized neural networks. 

\begin{figure}[htbp]
\centering
\subfigure[A 2-dimensional function $g$ which is the composition of $f$ and a linear function.]{
    \label{fig:pwl-r2}
    \includegraphics[width=5.5cm]{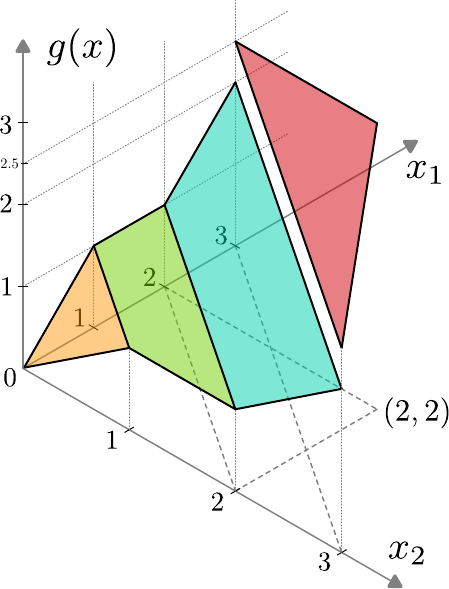}
    }
\quad
\subfigure[A piecewise linear function $f$ with the slope of every piece is either zero or one.]{
    \label{fig:staircase}
    \includegraphics[width=5.5cm]{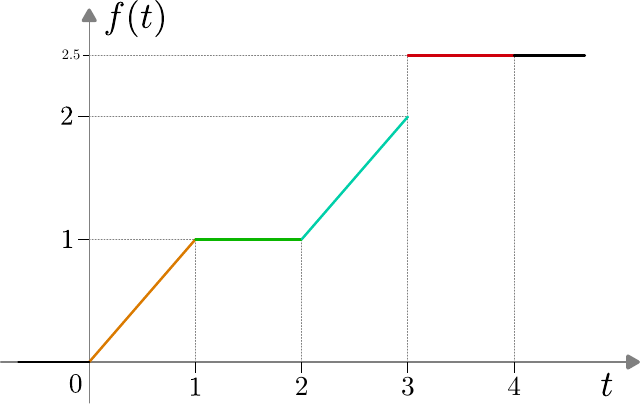}
}
\caption{\centering Figure \ref{fig:staircase} shows a univariate staircase function $f$ consisting of $4$ pieces.
Figure \ref{fig:pwl-r2} gives a graph of function $g(x) = f(w \cdot x)$, where $w = (1, 1)^T$}
\label{fig:main} 
\end{figure}

\section{Problem Formulation}
\label{problem formulation}
We consider a neural network $\mathcal{N}: \mathds{R}^{n_1} \rightarrow \mathds{R}^{n_2}$ with a total of $N$ neurons inclusive of the input, output and any hidden layers.
The neural network $\mathcal{N}$ maps an input $x \in \mathds{R}^{n_1}$ to an output $y \in \mathds{R}^{n_2}$ by the following series of computations:
\begin{equation}
\notag
    x_i = g_i(x_1, \dots, x_{i-1}) \coloneqq f_i \left( \sum_{j = 1}^{j=i-1} w_{i, j} x_j + b_i \right) \quad \forall i \in \{n_1+1, \dots, N\},
\end{equation}
where $y = (x_{N+1-n_2}, \dots, x_N)$ defines the output of $\mathcal{N}$, and $f_i$ denotes the activation function at neuron $i$ for $i \in \llbracket N \rrbracket$.
This description generalizes feedforward, convolutional, skip-layer, and recurrent neural network structures. \\

\noindent
Given a trained neural network ($w$ and $b$ are fixed), a constant vector $c \in \mathds{R}^{n_2}$, a fixed input $x_0$, a neighborhood $X_{\epsilon}(x_0)$ around $x_0$ and a real value $\xi \in \mathds{R}$, a generic verification task can be formulated as the decision problem:
\begin{equation}
\label{verification prob}
    \text{Is }\max_{x \in X_{\epsilon}(x_0)} c \cdot \mathcal{N}(x) \leq \xi?
\end{equation}

\noindent
Given a function $g: \mathds{R}^n \rightarrow \mathds{R}$, the graph of $g$ is defined to be gr$(g(x)) \coloneqq \{(x, y) \in \mathds{R}^{n+1} | y = g(x)\}$.
Formally, the verification problem \eqref{verification prob} can be reformulated as
\begin{subequations}
\label{ver prob}
    \begin{align} 
        \text{max } & c_1y_1 + \dots + c_{n_2}y_{n_2} \\
        \label{neuron model}
        & (x_1, \dots, x_i) \in \text{gr}(g_i(x_1, \dots, x_{i-1})) \quad \forall i \in \{n_1+1, \dots, N\} \\
        & (x_1, \dots, x_{n_1}) \in X_{\epsilon}(x_0) \\
        & y = (x_{N+1-n_2}, \dots, x_N).
    \end{align}
\end{subequations}
Generally, neural network verification is NP-complete \citep{katz2017reluplex}. 
Therefore, we approach this problem via mixed-integer programming.\\

\noindent
Because a mixed-integer programming formulation can only represent a closed set, we model the graph $g_i$ using its closure, i.e., we replace condition \eqref{neuron model} with
$$(x_1, \dots, x_i) \in \text{cl}(\text{gr}(g_i(x_1, \dots, x_{i-1}))) \quad \forall i \in \{n_1+1, \dots, N\},$$
where cl$(S)$ denotes the closure of a set $S$.
Since $\text{gr}(g_i) \subseteq \text{cl}(\text{gr}(g_i))$, we still have a relaxed verifier when replace the graph of $g_i$ with its closure.
When $g_i$ is continuous, $\text{gr}(g_i) = \text{cl}(\text{gr}(g_i))$, thus in this case, we do not adjust the feasible domain of \eqref{ver prob} at all.
Taking closure of a discontinuous function $g_i$ means that we allow $g_i$ to take two different values at discontinuity - a set with measure zero.
Hence, a natural approach to formulate \eqref{ver prob} as a MIP problem is to construct a MIP formulation for the closure of the graph of $g_i$ corresponding to each neuron.
Once we attain the MIP formulation, we can use it for both exact or inexact verifier.\\

\noindent
From now on, we focus on modeling an individual neuron; thus, we drop the neuron index subscript.
In particular, we denote $f$ as its activation function, $w \in \mathds{R}^n$ as its incoming weight, $b \in \mathds{R}$ as its bias, $x \in \mathds{R}^n$ as its input, and $y = g(x) \coloneqq f(w \cdot x + b)$ as its output. 
We assume that the input to a neuron lies in a box domain $D$, i.e., $x \in D \coloneqq \{x \in \mathds{R}^n | l \leq x \leq u\}$, where $l,u \in \mathds{R}^n$.
Since $D$ is bounded and $w$ is fixed, the value $L \coloneqq \min_{x \in D} w \cdot x + b$ and $U \coloneqq \max_{x \in D} w \cdot x + b$ are finite. 
Furthermore, we also assume that the activation function $f: [L, U] \rightarrow \mathds{R}$ is a univariate piecewise linear function with $k$ pieces, i.e.,
\begin{equation}
\label{eq:uni_pwl}
    f(t)= 
\begin{cases}
    a_1t + d_1 & \text{if } h_0 \leq t < h_1\\
    a_2t + d_2 & \text{if } h_1 \leq t < h_2\\
    & \vdots \\
    a_kt + d_k & \text{if } h_{k-1} \leq t \leq h_k,\\
\end{cases}
\end{equation}
where $L \equiv h_0 < h_1 < \dots < h_k \equiv U$ are its breakpoints.
In the form \eqref{eq:uni_pwl}, the function $f$ is
right continuous.
Indeed, for a general piecewise linear function, depending on how we decide its value, the function can be right or left-continuous at a breakpoint. 
However, since the value of $f$ at its breakpoints does not affect the closure of its graph, the optimization problem \eqref{ver prob} is invariant regardless of how we decide the continuity of $f$.
Thus, without loss of generality, we assume every piecewise linear function is of the form  \eqref{eq:uni_pwl}.
For now, we restrict $f$ to be a \emph{staircase} function and relax this constraint in Section \ref{General piecewise linear function}.
\begin{definition}
A univariate piecewise linear function $f: \mathds{R} \rightarrow \mathds{R}$ is a \textbf{staircase function} if there exists $s \in \mathds{R}$ such that $a_i \in \{0, s\}$ for every $i \in \llbracket k \rrbracket$.
\end{definition}
\noindent
We can see that a piecewise constant function is a staircase function with $a_i = 0$ for every $i \in \llbracket k \rrbracket$, and a ReLU is a staircase function with two pieces and $s = 1$.
Hence, staircase functions already capture quantized activation functions and ReLU, one of the most common activation functions in deep neural networks.
\\

\noindent
With the above assumptions, the function $g:D\rightarrow \mathds{R}$ corresponding to a neuron has the form
\[
    g(x)= f(w \cdot x + b) = 
\begin{cases}
    a_1 w\cdot x + \bar{d}_1 & \text{if } h_0 \leq w \cdot x + b < h_1\\
    a_2 w\cdot x + \bar{d}_2 & \text{if } h_1 \leq w \cdot x + b < h_2\\
    & \vdots \\
    a_k w \cdot x + \bar{d}_k & \text{if } h_{k-1} \leq w \cdot x + b \leq h_k,\\
\end{cases}
\]
where $\bar{d}_i \coloneqq a_ib + d_i$ for every $i \in \llbracket k \rrbracket$. As mentioned previously, we focus on constructing a MIP formulation for the closure of graph of $g$.
\begin{definition}
\label{formulation}
Let $S \subset \mathds{R}^{n+1}$ be a closed set, and $Q \subseteq \mathds{R}^{n+1 + \eta + k}$. We define the projection of $Q$ onto the first $n+1$ variables to be $\Pi_{x, y}(Q) \coloneqq \{(x, y) \in \mathds{R}^{n+1} \ | \ \exists p \in \mathds{R}^{\eta}, z \in \mathds{R}^k \text{ such that } (x,y,p,z) \in Q \}$.
Moreover, we say that:
\begin{enumerate}
    \item $Q$ is a \textbf{valid} MIP formulation of $S$ if $\Pi_{x, y}(Q \cap( \mathds{R}^{n+1+\eta)} \times \{0,1\}^k)) = S$.
    \item $Q$ is a \textbf{non-extended} MIP formulation of $S$ if $Q$ is a valid formulation and $\eta = 0$.
    \item $Q$ is an \textbf{ideal} MIP formulation if $Q$ is valid and {\normalfont ext}$(Q) \subset \mathds{R}^{n + 1 +\eta} \times \{0,1\}^{k}$. \\
\end{enumerate}
\end{definition}

\noindent
We consider the following formulation for the graph of $g$.
Let $D^i \coloneqq \{x \in D | \ h_{i-1} \leq w \cdot x + b \leq h_i \}$ be the portion of $D$ where $g(x)$ is an affine function with slope $a_iw$ and constant term $\bar{d}_i$. 
The \emph{Cayley embedding} \citep{huber2000cayley, vielma2018embedding, vielma2019small} for the closure of graph of $g$ is defined as:
$$S_{\text{Cayley}}(g) \coloneqq \bigcup_{i = 1}^k \{(x, y, z) | x \in D^i, \ y = g(x), \ z = e^i\},$$
where $e^i$ denotes the $i^{th}$ unit vector in $\mathds{R}^k$. 
In addition, for a set $Q$, we define the convex hull of $Q$, denoted as conv$(Q)$, to be the set consisting of all convex combinations of points in $Q$. Notationally, we denote
$$C(g) \coloneqq \text{conv}(S_{\text{Cayley}}(g)).$$
Trivially, $C(g)$ is a non-extended ideal formulation of cl$(\text{gr}(g)))$.
Since an ideal formulation offers the tightest possible relaxation \citep{vielma2015mixed}, $C(g)$ gives a strong formulation without any additional continuous variables. 
Thus, it is a good choice for modeling the function $g$.
Therefore, this article focuses on understanding the inequalities that define $C(g)$.
By extending a result of \citep[Proposition 4]{anderson2020strong} \footnote{Even though the proposition is stated, the proposition holds for any piecewise linear function.}, we can derive a description for $C(g)$ as
\begin{subequations}
\label{cayley}
\begin{align} 
        \label{original cayley upper}
        y  \leq \min_{\alpha \in \mathds{R}^n} \alpha \cdot x & + \sum_{i=1}^k (\max_{x^i \in D^i} (a_iw- \alpha) \cdot x^i + \bar{d}_i)z_i \\
        \label{original cayley lower}
        y \geq \max_{\alpha \in \mathds{R}^n} \alpha \cdot x &  + \sum_{i=1}^k (\min_{x^i \in D^i} (a_iw- \alpha) \cdot x^i + \bar{d}_i)z_i \\
        & (x, y, z) \in D \times \mathds{R} \times \Delta^k.
\end{align}
\end{subequations}

\noindent
The formulation \eqref{cayley} is not yet practical as the right-hand side in \eqref{original cayley upper} and \eqref{original cayley lower} contain multiple minimax or maximin problems.
On the other hand, the Cayley embedding $C(g)$ is a polyhedron; thus, there exists a finite number of linear constraints which describes $C(g)$.
However, $C(g)$ can require an exponential number of constraints.
For example, when the activation function is ReLU, \citep[Proposition 12]{anderson2020strong} showed that the number of facets is equal to $2^\eta$, where $\eta \coloneqq \text{supp}(w)$ is the number of non-zero coefficients of $w$.
Even in the case of binary activation, the Cayley embedding formulation can also need up to an exponential number of inequalities \citep{han2021single}. \\

\noindent
Therefore, to make the Cayley embedding formulation practical, we need a separation procedure.
We only present our transformation for \eqref{original cayley upper} since an analogous transformation can be done on \eqref{original cayley lower}, and we omit the details for brevity.\\

\noindent
For any fixed $\alpha \in \mathds{R}^n$, the dual of $\max_{x^i \in D^i} (a_iw-\alpha)'x^i$ is given by
\begin{subequations}
\begin{align}
    \text{min } & 
     u \cdot \beta^i - l \cdot \gamma^i + (h_i-b)\theta^i_1 - (h_{i-1}-b)\theta^i_2 \\
    \label{5b}
    \text{subject to } & [I_n, -I_n, w', -w']\begin{bmatrix}
    \beta^i \\ \gamma^i \\ \theta^i_1 \\ \theta^i_2
    \end{bmatrix} = a_i w - \alpha \\
    \label{5c}
    & \beta^i, \gamma^i \geq 0^n \\
    \label{5d}
    & \theta^i_1, \theta^i_2 \geq 0,
\end{align}
\end{subequations}
where $I_n$ denotes the identity matrix of size $n$.
Therefore, by strong duality, \eqref{original cayley upper} can be rewritten as
\begin{equation}
\label{intermediate step}
    \begin{split}
        y  & \leq \min_{\alpha \in \mathds{R}^n} \left(\alpha \cdot x + \sum_{i=1}^k z_i\min_{\beta, \gamma, \theta : \eqref{5b}, \eqref{5c}, \eqref{5d}} u \cdot \beta^i - l \cdot \gamma^i + (h_i-b)\theta^i_1 - (h_{i-1}-b)\theta^i_2 + \sum_{i=1}^k z_i\bar{d}_i\right).
    \end{split}
\end{equation}
For fixed $\hat{x} \in \mathds{R}^n$ and $\hat{z} \in \Delta^k$, the right-hand side of \eqref{intermediate step} is a linear programming problem where $\{\beta^i, \gamma^i, \theta^i_1, \theta^i_2\}_{i=1}^k$ and $\alpha$ are the decision variables. By dropping the constant term $\sum_{i=1}^k \hat{z}_i\bar{d}_i$ and taking the minimum over $\{\beta^i, \gamma^i, \theta^i_1, \theta^i_2\}_{i=1}^k$ and $\alpha$, we can write this linear programming problem as
\begin{subequations}
    \label{before change of var}
    \begin{alignat}{2}
        \text{min } & \sum_{i=1}^k \hat{z}_i(u \cdot \beta^i - l \cdot \gamma^i + (h_i-b)\theta^i_1 - (h_{i-1}-b)\theta^i_2) + && \hat{x} \cdot \alpha \\
        \text{subject to } & \beta^i_j - \gamma^i_j + w_j\theta^i_1 - w_j\theta^i_2 + \alpha_j = a_iw_j && \forall i \in \llbracket k \rrbracket j \in \llbracket n \rrbracket \\
        &\beta^i, \gamma^i, \theta^i_1, \theta^i_2 \geq 0 && \forall i \in \llbracket k \rrbracket.
    \end{alignat}
\end{subequations}
Note that we can solve \eqref{before change of var} to determine if a given point
$(\hat{x}, \hat{y}, \hat{z}) \in C(g)$.
If the optimal value is larger or equal to $\hat{y}$, then we say $(\hat{x}, \hat{y}, \hat{z}) \in C(g)$, otherwise, $(\hat{x},\hat{y},\hat{z}) \notin C(g)$ and we can retrieve a violated constraint based on the optimal value of $\alpha$.
This separation can be performed by an off-the-shelf LP solver.
However, the simplex method, a standard solver for linear programming problems, may take more than one iteration to prove optimality, where each simplex iteration can cost up to $O(n^2k^2)$ computational steps \citep[Section 3.3]{bertsimas1997introduction}.
Nevertheless, we can show that the separation procedure (or retrieving an optimal solution of \eqref{before change of var}) can be done in $O(n\text{log}n + \max(k, n))$ time complexity. \\

\noindent
Since $f$ is a staircase function, for every $i \in \llbracket k \rrbracket$, $a_i$ can only take value in $\{0, s\}$.
Assuming that $s > 0$, we scale the variables in \eqref{before change of var} as follows.
Let $\bar{\beta}^i_j \coloneqq \frac{\beta^i_j}{s|w_j|}$, $\bar{\gamma}^i_j \coloneqq \frac{\gamma^i_j}{s|w_j|}$, $\bar{\theta}^i_1 \coloneqq \frac{\theta^i_1}{s}$, $\bar{\theta}^i_2 \coloneqq \frac{\theta^i_2}{s}$, $\bar{\alpha}_j \coloneqq \frac{\alpha_j}{s|w_j|}$, and $\bar{w} \in \mathds{R}^n$ be defined by
\[
    \bar{w}_j \coloneqq 
\begin{cases}
    1& \text{if } w_j > 0\\
    -1& \text{if } w_j < 0 \\
    0 & \text{otherwise}.
\end{cases}
\]
With the above scaling of decision variables, \eqref{before change of var} can be reformulated as\footnote{If $s = 0$, we only need to scale $\bar{\beta}^i_j \coloneqq \frac{\beta^i_j}{|w_j|}$, $\bar{\gamma}^i_j \coloneqq \frac{\gamma^i_j}{|w_j|}$, and $\bar{\alpha}_j \coloneqq \frac{\bar{\alpha}_j}{|w_j|}$.
For the case $s < 0$, we can still derive \eqref{main dual} from \eqref{before change of var} by the same scaling of variables except we redefine \[
    \bar{w}_j \coloneqq 
\begin{cases}
    -1& \text{if } w_j > 0\\
    1& \text{if } w_j < 0 \\
    0 & \text{otherwise}.
\end{cases}
\]}
\begin{subequations}
    \label{main dual}
    \begin{alignat}{2}
        \label{objective}
        \text{min } & s\sum_{i=1}^k z_i( \sum_{j=1}^n u_j |w_j| \bar{\beta}^i_j - \sum_{j=1}^n l_j |w_j| \bar{\gamma}^i_j + (h_i-b)\bar{\theta}^i_1 - && (h_{i-1} -b)\bar{\theta}^i_2) + s \sum_{j=1}^n x_j |w_j| \bar{\alpha}_j \\
        \text{subject to } & \bar{\beta}^i_j -\bar{\gamma}^i_j + \bar{w}_j\bar{\theta}^i_1 - \bar{w}_j\bar{\theta}^i_2 + \bar{\alpha}_j = \frac{a_i}{s}\bar{w}_j && \forall i \in \llbracket k \rrbracket, j \in \llbracket n \rrbracket \\
        &\bar{\beta}^i, \bar{\gamma}^i, \bar{\theta}^i_1, \bar{\theta}^i_2 \geq 0 && \forall i \in \llbracket k \rrbracket.
    \end{alignat}
\end{subequations}
By definition, $\bar{w} \in \{0, \pm1\}^n$. Since each $a_i \in \{0, s\}$, we infer that $\frac{a_i}{s} \in \{0, 1\}$ for each $i \in \llbracket k \rrbracket$. 
Moreover, the constraint matrix of \eqref{main dual} has every component being $0$ or $\pm 1$. If we denote $A \coloneqq [I_n, -I_n, \bar{w}', -\bar{w}'] \in \mathds{R}^{n \times (2n+2)}$, the feasible domain of \eqref{main dual}, denoted as $\bm{P}$, can be written in matrix form as follows:
\begin{equation}
\label{polyhedron}
\underbrace{
\begin{bmatrix}
A & 0 & \dots & 0 & I_n \\
0 & A & \dots & 0 & I_n \\
\vdots & \vdots & \ddots & \vdots & \vdots \\
0 & 0  & \dots & A & I_n 
\end{bmatrix}}_{\coloneqq \hat{A}}
\begin{bmatrix}
\bar{\beta}^1 \\ \bar{\gamma}^1 \\ \bar{\theta}^1 \\ \vdots \\ \bar{\beta}^k \\ \bar{\gamma}^k \\ \bar{\theta}^k \\ \bar{\alpha}
\end{bmatrix} =
\begin{bmatrix}
\frac{a_1}{s} \bar{w} \\ \frac{a_2}{s} \bar{w} \\ \frac{a_3}{s} \bar{w} \\ \vdots \\ \frac{a_k}{s} \bar{w}
\end{bmatrix}, \text{ and }
\begin{bmatrix}
\bar{\beta}^1 \\ \bar{\gamma}^1 \\ \bar{\theta}^1 \\ \vdots \\ \bar{\beta}^k \\ \bar{\gamma}^k \\ \bar{\theta}^k
\end{bmatrix} \geq \textbf{0}.
\end{equation}

\section{Polyhedral Results}
\label{extreme}
From the previous section, we can derive a separating hyperplane of a fixed point $\hat{x}, \hat{z}$ based on an optimal solution of \eqref{main dual}. 
Notably, the feasible domain $\bm{P}$ defined by the system \eqref{polyhedron} has the left-hand side matrix's entries in $\{0,\pm1\}$.
Interestingly, we can show that $\hat{A}$ is totally unimodular \citep[Chapter 4]{conforti2014integer}, and thus all extreme points of $\bm{P}$ are are integral.
However, we can derive stronger properties of $\bm{P}$, that is, not only its extreme points are integral, but all of its vertices and extreme rays are also $\{0, \pm1\}$-vectors
\footnote{Since two extreme rays are equivalent if one is a positive multiple of the other, here when we say that extreme rays of $\bm{R}$ are $\{0, \pm1\}$-vectors, we actually mean that every extreme ray of $\bm{R}$ is equivalent to a $\{0, \pm1\}$-extreme ray.}.
These characterizations of extreme points and extreme rays of $\bm{P}$ will play an important role in developing a fast algorithm for solving \eqref{main dual}, so we spend the remainder of this section proving these properties.

\subsection{Extreme Rays}
We first begin with a characterization of the extreme rays. 
Extreme rays of $\bm{P}$ are extreme rays of its recession cone $\bm{R}$, which is defined to be the non-zero solutions of the following system:
\begin{equation}
\label{recession cone}
\begin{bmatrix}
A & 0 & \dots & 0 & I_n \\
0 & A & \dots & 0 & I_n \\
\vdots & \vdots & \ddots & \vdots & \vdots \\
0 & 0  & \dots & A & I_n 
\end{bmatrix}
\begin{bmatrix}
\bar{\beta}^1 \\ \bar{\gamma}^1 \\ \bar{\theta}^1 \\ \vdots \\ \bar{\beta}^k \\ \bar{\gamma}^k \\ \bar{\theta}^k \\ \bar{\alpha}
\end{bmatrix} =
\mathbf{0}, \text{ and }
\begin{bmatrix}
\bar{\beta}^1 \\ \bar{\gamma}^1 \\ \bar{\theta}^1 \\ \vdots \\ \bar{\beta}^k \\ \bar{\gamma}^k \\ \bar{\theta}^k \\ \bar{\alpha}
\end{bmatrix} \geq \textbf{0}.
\end{equation}
We will show that the extreme rays of $\bm{R}$ are $\{0,\pm1\}$-vectors. 
We transform \eqref{recession cone} to its standard form representation and then artificially introduce upper bounds on each of our decision variables,:
\begin{equation}
\label{standard bounded cone}
\begin{bmatrix}
A & 0 & \dots & 0 & I_n & -I_n \\
0 & A & \dots & 0 & I_n & -I_n \\
\vdots & \vdots & \ddots & \vdots & \vdots & \vdots \\
0 & 0  & \dots & A & I_n  & -I_n
\end{bmatrix}
\begin{bmatrix}
\bar{\beta}^1 \\ \bar{\gamma}^1 \\ \bar{\theta}^1 \\ \vdots \\ \bar{\beta}^k \\ \bar{\gamma}^k \\ \bar{\theta}^k \\ \bar{\alpha}_+ \\ \bar{\alpha}_-
\end{bmatrix} =
\mathbf{0}, \text{ and }
\textbf{0} \leq 
\begin{bmatrix}
\bar{\beta}^1 \\ \bar{\gamma}^1 \\ \bar{\theta}^1 \\ \vdots \\ \bar{\beta}^k \\ \bar{\gamma}^k \\ \bar{\theta}^k \\ \bar{\alpha}_+ \\ \bar{\alpha}_-
\end{bmatrix}
\leq \textbf{1}.
\end{equation}
We first prove that every extreme point of the polyhedron defined by \eqref{standard bounded cone} is a $\{0,1\}$-vector, then we will use this result to prove that every extreme ray of $\bm{R}$ is a $\{0, \pm1\}$-vector.
\begin{lemma}
\label{extreme rays}
Let $(\hat{\beta}^1, \hat{\gamma}^1, \hat{\theta}^1, \dots, \hat{\beta}^k, \hat{\gamma}^k, \hat{\theta}^k, \hat{\alpha}_+, \hat{\alpha}_-)$ be an extreme point of the polyhedron defined by \eqref{standard bounded cone}. 
Then $\hat{\beta}^i, \hat{\gamma}^i \in \{0, 1\}^n$, $\hat{\theta}^i \in \{0,1\}^2$ for each $ i \in \llbracket k \rrbracket$, and $\hat{\alpha}_+, \hat{\alpha}_- \in \{0, 1\}^n$.
\end{lemma}
\proof{Proof. }
We will prove by induction on the number of pieces $k$ that every basic feasible solution (BFS) of \eqref{standard bounded cone} satisfies the desired condition. For the base case where $k = 1$, the feasible domain of \eqref{standard bounded cone} is given by:
\[
\begin{bmatrix}
I_n & -I_n & \bar{w} & -\bar{w} & I_n & -I_n
\end{bmatrix}
\begin{bmatrix}
\beta^1 \\ \gamma^1 \\ \theta^1 \\ \alpha_+ \\ \alpha_-
\end{bmatrix}
= 
\mathbf{0}, \\
\]
$$
0^n \leq \beta^1, \gamma^1 \leq 1^n, \ 0 \leq \theta^1_1, \theta^1_2 \leq 1, \ 
0^n \leq \alpha_+, \alpha_- \leq 1^n.
$$
Let $B \in \mathds{R}^{n \times n}$ be the basis matrix of an arbitrary BFS.
Since $B$ is non-singular, there cannot be two columns of $B$ that are equal to either $\bar{w}$ or $-\bar{w}$. 
Hence, $B$ can have at most one column equal to $\bar{w}$ or $-\bar{w}$, and each other column must equal to a column of either $I_n$ or $-I_n$. 
Because columns of $B$ are linearly independent, by some columns permutation and scaling some columns by $-1$, we can transform $B$ into
\[
\overline{B} = 
\begin{bmatrix}
I_{n-1} & \Tilde{w} \\
0 & 1
\end{bmatrix},
\]
where $\Tilde{w} \in \{0, \pm1\}^{n-1}$ because $\bar{w} \in \{0, \pm1\}^n$. Since $\text{det}(\overline{B}) = 1$, we have  $\text{det}(B) = \pm1$. 
Hence, for $k = 1$, every BFS of the polyhedron defined by \eqref{standard bounded cone} is integral, and due to the bound constraints, its components must be in $\{0, 1\}$.\\

\noindent
By induction, assume that the statement in Lemma \ref{extreme rays} is true for $k-1$, for some $k \geq 2$. 
We will prove that it also holds for $k$. For the rest of the proof, with a slight abuse of notation, we sometimes treat a matrix as a set consisting of its columns. \\

\noindent
Let $(\hat{\beta}^1, \hat{\gamma}^1, \hat{\theta}^1, \dots, \hat{\alpha}_+, \hat{\alpha}_-)$ be an extreme point of \eqref{standard bounded cone}. 
Furthermore, we denote $p^i_B$ for each $i \in \llbracket k \rrbracket$ and $\alpha_{B}$ as the basic variables of $\{\beta^i_1, \dots,  \beta^i_n, \gamma^i_1, \dots, \gamma^i_n, \theta^i_1, \theta^i_2\}$ and $\{\alpha_{+, 1}, \dots, \alpha_{+, n}, \alpha_{-, 1}, \dots, \alpha_{-, n}\}$, respectively. 
When every non-basic variable is set to be either $0$ or $1$, the basic variables need to satisfy the following system of equations:
\[
\underbrace{
\begin{bmatrix}
B^1 & \dots & 0 & B^0 \\
\vdots & \ddots & \vdots & \vdots \\
0 &  \dots & B^k & B^0
\end{bmatrix}}_B
\begin{bmatrix}
p^i_B \\ \vdots \\ p^i_B \\ \alpha_{B}
\end{bmatrix} = b,
\]
where each $B^i$ for $i \in \llbracket k \rrbracket$ is a sub-column matrix of $A$ corresponding to the basic variables $p^i_B$, and $B^0$ is a sub-column matrix of $[I_n, -I_n]$ corresponding to basic variables $\alpha_{B}$. 
The right-hand side vector $b$ is defined by the value of the non-basic variables. 
Since every non-basic variables are constrained to take values in $\{0,1\}$, the vector $b$ has $\{0, \pm1\}$ entries. \\

\noindent
\textbf{Case 1:} There exists $i \in \llbracket k \rrbracket$ such that $B^i$ does not have $\bar{w}$ or $-\bar{w}$ as one of its columns. 
Without loss of generality, we assume that $B^1$ is such a matrix. 
Notationally, we denote $I(B^i) = \{j| \exists j' \text{ such that } \ B^i_{j'} = e^j \text{ or } B^i_{j'} = -e^j\}$, and we call a column $B^i_{j'}$ of $B^i$ an identity column if there exists $j \in \llbracket n \rrbracket$ such that $B^i_{j'} = e^j$ or $B^i_{j'} = e^j$. 
We must have $I(B^1) \cup I(B^0) = \llbracket n \rrbracket$, because otherwise the basis matrix $B$ contains a $0$ row, which contradicts the fact that $B$ is non-singular. 
Let $\bar{B}^0 = \{ B^0_{j'} | \exists j \in I(B^1) \cap I(B^0) \text{ such that } B^0_{j'} = e^j \text{ or } B^0_{j'} = -e^j \}$ be the sub-column matrix of $B^0$ that contains all the common identity column of $B^0$ and $B^1$, and $\hat{B}^0 = \{ B^0_{j'} | \exists j \in I(B^1) \setminus I(B^0) \text{ such that } B^0_{j'} = e^j \text{ or } B^0_{j'} = -e^j \}$ be the sub-column matrix of $B^0$ consisting of all identity column exclusive to $B^0$.
We can observe that every basic variable corresponding to $\hat{B}^0$ is determined from the first $n$ equalities.
The basic variable $p^2_B, \dots p^k_B, \alpha_{\bar{B}}$, where $\alpha_{\bar{B}}$ is subset of $\alpha_B$ but only takes the variables that corresponds to $\bar{B}^0$, need to satisfy the following system
\[
\begin{bmatrix}
B^2 & \dots & 0 & \bar{B}^0 \\
\vdots & \ddots & \vdots & \vdots \\
0 & \dots & B^k & \bar{B}^0
\end{bmatrix}
\begin{bmatrix}
p^2_B \\ \vdots \\ p^k_B \\ \alpha_{\bar{B}}
\end{bmatrix} = b.
\]
This is a $(k-1)n \times (k-1)n$ linear system that must uniquely define $p^2_B, \dots, p^k_B, \alpha_{\bar{B}}$ otherwise it would violate the non-singularity of $B$. 
Hence, $(p^2_B, \dots, p^k_B, \alpha_{\bar{B}})$ is a basis solution of \eqref{standard bounded cone} where $k$ is replaced by $k-1$. Its feasibility comes from the fact that $p^2_B, \dots, p^k_B, \alpha_{\bar{B}}$ are between $0$ and $1$. Thus, by induction hypothesis, $p^2_B, \dots, p^k_B, \alpha_{\bar{B}}$ have entries being either $0$ or $1$. Finally, because $\alpha_B$ are integral, $B^1$ only consists of unit vectors, $p^1_B$ must also be integral, and hence $p^1_B$ is $\{0,1\}$-vector.\\

\noindent
\textbf{Case 2:} Every $B^i$ has $\bar{w}$ or $-\bar{w}$ as one its column, which means, for each $i \in \llbracket k \rrbracket$, between $\{\theta^i_1, \theta^i_2\}$, exactly one of them is a basic variable while the other is non-basic.
Let $I_1 = \{i \in ~\llbracket k \rrbracket \ | \ \theta^i_1 \text{ is basic variable} \}$ and $I_2 = \{i \in \llbracket k \rrbracket | \ \theta^i_2 \text{ is basic variable} \}$.
Trivially, we have that $I_1 \cap I_2 = \emptyset$ and $I_1 \cup I_2 = \llbracket k \rrbracket$. In what follows, we will prove that, at least one of the variables from the set $I_1 \cup I_2$ must take values of $0$ or $1$. 
Suppose otherwise, then all of variables from $I_1 \cup I_2$ has fractional value. 
To derive a contradiction, we will show that the BFS $(\hat{\beta}^1, \hat{\gamma}^1, \hat{\theta}^1, \dots, \hat{\alpha}_+, \hat{\alpha}_-)$ with basis $B$ is not an extreme point, i.e., it is a convex combination of two other feasible solutions. We denote
$$Z = \{j \in \llbracket n \rrbracket | \ \alpha_j =  \alpha_{+,j} - \alpha_{-, j} \in \{0, \pm 1\}\}$$
$$Q = \{j \in \llbracket n \rrbracket | \ \alpha_j =  \alpha_{+,j} - \alpha_{-, j} \in (-1, 1) \setminus \{0\} \}.$$
For any $j \in Z$, and for any $i \in \llbracket k \rrbracket$, from \eqref{standard bounded cone}, we have
$$\beta^i_j - \gamma^i_j + \bar{w}_j \theta^i_1  - \bar{w}_j \theta^i_2 + \alpha_{+j} - \alpha_{-j} = 0.$$
Since exactly one of $\theta^i_1$ and $\theta^i_2$ is fractional, either $\beta^i_j$ or $\gamma^i_j$ (or both) must be fractional. We define a mapping $\phi(i, j)$ to a variable, where $i \in \llbracket k \rrbracket$ and $j \in \llbracket n \rrbracket$, as follows:
\begin{enumerate}
    \item If $j \in Z$, then $\phi(i, j)$ maps to $\gamma^i_j$ or $\beta^i_j$, whichever has fractional value (only refer to one, not both variables).
    \item If $j \in Q$, then $\phi(i, j)$ maps to $\alpha_{+j}$ or $\alpha_{-j}$, whichever has fractional value.
\end{enumerate}
As the function $\phi(i, j)$ picks out an fractional variable for each $i \in \llbracket k \rrbracket$ and $j \in \llbracket n \rrbracket$, we can choose an $\epsilon$ as follows:
$$1 > \epsilon \coloneqq \min \{\theta^i_1, 1 - \theta^i_1, \ \forall i \in I_1, \theta^i_2, 1 - \theta^i_2, \ \forall i \in I_2, |\phi(i, j)|, 1 - |\phi(i, j)|, \forall \ i, j\} > 0$$
Let $(\bar{\beta^1}, \bar{\gamma}^1, \bar{\theta}^1, \dots, \bar{\beta}^k, \bar{\gamma}^k, \bar{\theta}^k, \bar{\alpha}_+, \bar{\alpha}_-)$ be the same as $(\hat{\beta}^1, \hat{\gamma}^1, \hat{\theta}^1, \dots, \hat{\beta}^k, \hat{\gamma}^k, \hat{\theta}^k, \hat{\alpha}_+, \hat{\alpha}_-)$ except
\begin{enumerate}
    \item $\bar{\theta}^i_1 := \hat{\theta}^i_1 + \epsilon$ for all $i \in I_1$
    \item $\bar{\theta}^i_2 := \hat{\theta}^i_2 - \epsilon$ for all $i \in I_2$
    \item For $j \in Q$, if $\bar{w}_j = 1$, then $\bar{\alpha}_{-j}:= \hat{\alpha}_{-j} + \epsilon$, otherwise if $\bar{w}_j = -1$, then $\bar{\alpha}_{+j} := \hat{\alpha}_{+j} + \epsilon$
    \item For $j \in Z$, then we consider 4 cases
    \begin{itemize}
        \item $f(i, j) = \beta^i_j$ and $\bar{w}_j = 1$, then $\bar{\beta}^i_j := \hat{\beta}^i_j - \epsilon$
        \item $f(i, j) = \beta^i_j$ and $\bar{w}_j = -1$, then $\bar{\beta}^i_j := \hat{\beta}^i_j + \epsilon$
        \item $f(i, j) = \gamma^i_j$ and $\bar{w}_j = 1$, then $\bar{\gamma}^i_j := \hat{\gamma}^i_j + \epsilon$
        \item $f(i, j) = \gamma^i_j$ and $\bar{w}_j = -1$, then $\bar{\gamma}^i_j := \hat{\gamma}^i_j - \epsilon$
    \end{itemize}
\end{enumerate}
Because of our choice $\epsilon$, $(\bar{\beta^1}, \bar{\gamma}^1, \bar{\theta}^1, \dots, \bar{\beta}^k, \bar{\gamma}^k, \bar{\theta}^k, \bar{\alpha}_+, \bar{\alpha}_-)$ is feasible. Similarly, we pick a feasible solution $(\Tilde{\beta^1}, \Tilde{\gamma}^1, \Tilde{\theta}^1, \dots, \Tilde{\beta}^k, \Tilde{\gamma}^k, \Tilde{\theta}^k, \bar{\alpha}_+, \bar{\alpha}_-)$ be the same as $(\hat{\beta}^1, \hat{\gamma}^1, \hat{\theta}^1, \dots, \hat{\beta}^k, \hat{\gamma}^k, \hat{\theta}^k, \hat{\alpha}_+, \hat{\alpha}_-)$ except
\begin{enumerate}
    \item $\Tilde{\theta}^i_1 := \theta^i_1 - \epsilon$ for all $i \in I_1$
    \item $\Tilde{\theta}^i_2 := \theta^i_2 + \epsilon$ for all $i \in I_2$
    \item For $j \in Q$, if $w_j = 1$, then $\bar{\alpha}_{+j}:= \alpha_{+j} - \epsilon$, otherwise if $w_j = -1$, then $\bar{\alpha}_{-j} := \alpha_{-j} - \epsilon$
    \item For $j \in Z(\alpha)$, then we consider 4 cases
    \begin{itemize}
        \item $f(i, j)$ is $\beta^i_j$ and $w_j = 1$, then $\Tilde{\beta}^i_j := \beta^i_j + \epsilon$
        \item $f(i, j)$ is $\beta^i_j$ and $w_j = -1$, then $\Tilde{\beta}^i_j := \beta^i_j - \epsilon$
        \item $f(i, j)$ is $\gamma^i_j$ and $w_j = 1$, then $\Tilde{\gamma}^i_j := \gamma^i_j - \epsilon$
        \item $f(i, j)$ is $\gamma^i_j$ and $w_j = -1$, then $\Tilde{\gamma}^i_j := \gamma^i_j + \epsilon$
    \end{itemize}
\end{enumerate}
We have $$(\hat{\beta}^1, \hat{\gamma}^1, \hat{\theta}^1, \dots, \hat{\beta}^k, \hat{\gamma}^k, \hat{\theta}^k, \hat{\alpha}_+, \hat{\alpha}_-) = \frac{1}{2}((\bar{\beta^1}, \bar{\gamma}^1, \bar{\theta}^1, \dots, \bar{\beta}^k, \bar{\gamma}^k, \bar{\theta}^k, \bar{\alpha}_+, \bar{\alpha}_-) + $$ $$(\Tilde{\beta^1}, \Tilde{\gamma}^1, \Tilde{\theta}^1, \dots, \Tilde{\beta}^k, \Tilde{\gamma}^k, \Tilde{\theta}^k, \bar{\alpha}_+, \bar{\alpha}_-)),$$ which contradicts the fact that $(\beta^1, \gamma^1, \theta^1, \dots, \beta^k, \gamma^k, \theta^k, \alpha_+, \alpha_-)$ is a BFS. Thus, there must be $i$ such that $\theta^i_1$ and $\theta^i_2$ both equal to $0$ or $1$. Since $B^i$ already contains a column of $\bar{w}$ or $-\bar{w}$, it cannot contain every unit vector in $\mathds{R}^n$, we can pick another column that is unit vector as basic variable and swap out $\theta^i_1$ (or $\theta^i_2$), with every other non-basic variables remain the same, the new basis determines the same BFS, and in this case, one of $B^i$ does not contain a column being $\bar{w}$ or $-\bar{w}$. Hence, by case 1, $(\beta^1, \gamma^1, \theta^1, \dots, \beta^k, \gamma^k, \theta^k, \alpha_+, \alpha_-)$ is a $\{0, \pm1\}$-vector.
\endproof
\begin{corollary}
\label{corollary 1}
Let $(\hat{\beta}^1, \hat{\gamma}^1, \hat{\theta}^1, \dots, \hat{\alpha})$ be an extreme ray of $\bm{P}$ whose sup-norm is $1$. Then it has every component taking value in $\{0, \pm1\}$.
\end{corollary}
\proof{Proof.}
By contradiction, suppose that $(\hat{\beta}^1, \hat{\gamma}^1, \hat{\theta}^1, \dots, \hat{\alpha})$ is an extreme ray of $\bm{P}$ with sup-norm $1$ and there exists one of its component takes value in $(-1, 1) \setminus 0$. Let $\alpha_+ = \max (\hat{\alpha}, 0)$ and $\alpha_- = \min (-\hat{\alpha}, 0)$, we have $(\hat{\beta}^1, \hat{\gamma}^1, \hat{\theta}^1, \dots, \alpha_+, \alpha_-)$ is a feasible solution of \eqref{standard bounded cone}. Since $(\hat{\beta}^1, \hat{\gamma}^1, \hat{\theta}^1, \dots, \alpha_+, \alpha_-)$ has one component takes value in $(-1, 1) \setminus \{0\}$, by Lemma \ref{extreme rays}, it is not an extreme point of the polyhedron defined by \eqref{standard bounded cone}. 
Thus, there exist two distinct feasible solutions of \eqref{standard bounded cone} denoted as $(\bar{\beta}^1, \bar{\gamma}^1, \bar{\theta}^1, \dots, \bar{\alpha}_+, \bar{\alpha}_-), (\Tilde{\beta}^1, \Tilde{\gamma}^1, \Tilde{\theta}^1, \dots, \bar{\alpha}_+, \bar{\alpha}_-)$ and $(\hat{\beta}^1, \hat{\gamma}^1, \hat{\theta}^1, \dots, \alpha_+, \alpha_-)$ is the midpoint of the segment connected by these two points. 
Moreover, because at least one component of $(\hat{\beta}^1, \hat{\gamma}^1, \hat{\theta}^1, \dots, \alpha_+, \alpha_-)$ must take values $1$ or $-1$, the same component of $(\bar{\beta}^1, \bar{\gamma}^1, \bar{\theta}^1, \dots, \bar{\alpha}_+, \bar{\alpha}_-)$ and  $(\Tilde{\beta}^1, \Tilde{\gamma}^1, \Tilde{\theta}^1, \dots, \bar{\alpha}_+, \bar{\alpha}_-)$ also take the value $1$. Hence, there does not exist $\lambda > 0$ such that  $(\bar{\beta}^1, \bar{\gamma}^1, \bar{\theta}^1, \dots, \bar{\alpha}_+ - \bar{\alpha}_-) = \lambda (\hat{\beta}^1, \hat{\gamma}^1, \hat{\theta}^1, \dots, \hat{\alpha})$ while we have
$$(\hat{\beta}^1, \hat{\gamma}^1, \hat{\theta}^1, \dots, \hat{\alpha}) = \frac{1}{2} (\bar{\beta}^1, \bar{\gamma}^1, \bar{\theta}^1, \dots, \bar{\alpha}_+ - \bar{\alpha}_-) + \frac{1}{2} (\Tilde{\beta}^1, \Tilde{\gamma}^1, \Tilde{\theta}^1, \dots, \bar{\alpha}_+ - \bar{\alpha}_-), $$
which contradicts the fact that $(\hat{\beta}^1, \hat{\gamma}^1, \hat{\theta}^1, \dots, \hat{\alpha})$ is an extreme ray of $\bm{P}$.
\endproof

\subsection{Extreme Points}
We have shown that the extreme rays of $\bm{P}$ are $\{0, \pm1\}$-vectors. Next, we are going to prove the same result for extreme points of $\bm{P}$. Similarly, instead of working with system \eqref{polyhedron} directly, we will work with its standard form:
\begin{equation}
\label{standard cone}
\begin{bmatrix}
A & 0 &  \dots & 0 & I_n & -I_n \\
0 & A &  \dots & 0 & I_n & -I_n \\
\vdots &  \vdots & \ddots & \vdots & \vdots & \vdots \\
0 & 0 &  \dots & A & I_n &  -I_n
\end{bmatrix}
\begin{bmatrix}
\bar{\beta}^1 \\ \bar{\gamma}^1 \\ \bar{\theta}^1 \\ \vdots \\ \bar{\beta}^k \\ \bar{\gamma}^k \\ \bar{\theta}^k \\ \bar{\alpha}_+ \\ \bar{\alpha}_-
\end{bmatrix} =
\begin{bmatrix}
\frac{a_1}{s}\bar{w} \\ \vdots \\ \frac{a_k}{s}\bar{w}
\end{bmatrix}, \quad
\bar{\beta}^i_j, \bar{\gamma}^i_j, \bar{\theta}^i_1, \bar{\theta}^i_2 \bar{\alpha}_+, \bar{\alpha}_- \geq 0 \ \forall i, j.
\end{equation}
\begin{lemma}
\label{extreme points}
Let $(\hat{\beta}^1, \hat{\gamma}^1, \hat{\theta}^1, \dots, \hat{\beta}^k, \hat{\gamma}^k, \hat{\theta}^k, \hat{\alpha}_+, \hat{\alpha}_-)$ be an extreme point of the polyhedron defined by \eqref{standard cone}. Then
$\beta^i, \gamma^i \in \{0,1\}^n$, $\theta^i \in \{0,1\}^2$ for each $i \in \llbracket k \rrbracket$, and $\alpha_+, \alpha_- \in \{0, 1\}^n$.
\end{lemma}
\proof{Proof.}
Similar to the proof of Lemma \ref{extreme rays}, we denote $B$ as the basis matrix for the BFS solution $(\hat{\beta}^1, \hat{\gamma}^1, \hat{\theta}^1, \dots, \hat{\beta}^k, \hat{\gamma}^k, \hat{\theta}^k, \hat{\alpha}_+, \hat{\alpha}_-)$ of \eqref{standard cone}. 
Again, we will prove by induction on $k$. 
For the base case where $k = 1$, since the RHS of equality constraints in \eqref{standard cone} is either $0$ or $\bar{w}$, if $B$ has a column equal to $\bar{w}$ or $-\bar{w}$, then its corresponding basic solution is the vector which has every component equal to $0$ except $\theta^1_1 = 1$ or $\theta^1_2 = -1$. If $B$ does not contain a row of $\bar{w}$ nor $-\bar{w}$, then $B$ is equivalent to identity matrix by some columns permutation (and possibly scaling some columns by $-1$). Thus for the base case $k = 1$, every extreme point is a \{0,1\}-vector.\\

\noindent
Same as the proof of Lemma \ref{extreme rays}, for the extreme point $(\hat{\beta}^1, \hat{\gamma}^1, \hat{\theta}^1, \dots, \hat{\beta}^k, \hat{\gamma}^k, \hat{\theta}^k, \hat{\alpha}_+, \hat{\alpha}_-)$ of the polyhedron defined by \eqref{standard cone}, we denote $p^1_B, \dots, p^k_B, \alpha_B$ as the basic variable. The basis variables need to satisfy the following system of equations:
\[
\underbrace{
\begin{bmatrix}
B^1 & 0 & \dots & 0 & B^0 \\
0 & B^2 & \dots & 0 & B^0 \\
\vdots & \vdots & \ddots & \vdots & \vdots \\
0 & 0 & \dots & B^k & B^0
\end{bmatrix}
}_B
\begin{bmatrix}
p^1_B \\ p^2_B \\ \vdots \\ p^k_B \\ \alpha_B
\end{bmatrix} = 
\begin{bmatrix}
\frac{a_1}{s}\bar{w} \\ \frac{a_2}{s}\bar{w} \\ \frac{a_3}{s}\bar{w} \\ \vdots \\ \frac{a_k}{s}\bar{w}
\end{bmatrix},
\]
where $B^i$ for $i = 1, \dots, k$ is sub-column matrix of $A$, $B^0$ is sub-column matrix of $[I_n, -I_n]$. \\

\noindent
\textbf{Case 1:} There exists $i \in \llbracket k \rrbracket$ such that $B^i$ does not contain a column of $\bar{w}$ or $-\bar{w}$. The proof for this case is similar to the proof for Case 1 of Lemma \ref{extreme rays} using the induction hypothesis. \\

\noindent
\textbf{Case 2:} Every $B^i$ has $\bar{w}$ or $-\bar{w}$ as one of its column. Then, Since the RHS is a vector which is obtained from concatenating $\bar{w}$ or $0^n$, we can easily obtain a solution by setting $\theta^i_1 = 1$ or $\theta^i_2 = -1$ and the rest is $0$.
\endproof

\noindent
Based on Lemma \ref{extreme points}, we can derive the following corollary for the extreme rays of $\mathbf{P}$.
\begin{corollary}
\label{corollary 2}
Every extreme point of \eqref{main dual} is $\{0, \pm1\}$-vector.
\end{corollary}

\section{Separation Procedure}
\label{sep procedure}
Returning to our description for $C(g)$ as shown in equation \eqref{cayley}, we have shown that for a fixed $(\hat{x}, \hat{y}, \hat{z})$, the RHS of \eqref{original cayley upper} and \eqref{original cayley lower} are in fact linear programs.
After scaling the variables, we showed that these linear programs are equivalent to \eqref{main dual}.
Since the feasible domain of \eqref{main dual} is always non-empty and pointed, there can only be two scenarios: either \eqref{main dual} is unbounded or has a finite objective value.
Hence, there is either an extreme ray with a negative objective cost or an extreme point with a finite optimal value.
Thus, if we denote ext$(P)$ as the set consisting of extreme rays and extreme points of $\bm{P}$, and $\Lambda \coloneqq \{\alpha \in \mathds{R}^n \ | \ \alpha = s|w|\bar{\alpha} \text{ such that } \exists (\bar{\beta}^1, \bar{\gamma}^1, \bar{\theta}^1, \dots, \bar{\alpha}) \in \text{ext}(\bm{P})\}$, the inequalities \eqref{original cayley upper} can be replaced by
\begin{equation}
\label{lambda}
    y \leq \alpha \cdot x + \sum_{i=1}^k (\max(a_iw - \alpha)\cdot x^i + \bar{d}_i)z_i \quad \forall \alpha \in \Lambda,
\end{equation}
and similar for \eqref{original cayley lower}.
Certainly, we want to avoid enumerating all extreme points and extreme rays of a polyhedral since there can be an exponential number of them.
Therefore, for a fixed $(\hat{x}, \hat{y}, \hat{z})$, we solve \eqref{main dual} to either retrieve an $\alpha$ that corresponds with a violated constraints for $(\hat{x}, \hat{y}, \hat{z})$ or a proof that $(\hat{x}, \hat{y}, \hat{z})$ is feasible.
In this section, we show that the separation for $(\hat{x}, \hat{y}, \hat{z})$ can be done in super linear time complexity.
Because of Corollaries \ref{corollary 1} and \ref{corollary 2}, in order to find an optimal solution of \eqref{main dual}, we only need to restrict our search to $\{0, \pm1\}$-vectors. 

\subsection{Unboundedness}
We first try to answer if \eqref{main dual} has an unbounded objective value. 
By Corollary \ref{corollary 1}, problem \eqref{main dual} is unbounded from below if and only if there is a $\{0, \pm1\}$-ray with negative cost. 
Hence, we will look for a $\{0, \pm1\}$-vector belonging to the recession cone $\bm{R}$ with negative objective cost \eqref{objective}.
Proposition \label{gamma beta theta} gives a structural property of such a ray.

\begin{proposition}
\label{gamma beta theta}
If \eqref{main dual} is unbounded from below, and every ray of $\bm{R}$ where $\bar{\alpha} = 0$ has a non-negative objective cost \eqref{objective}, then there exists a $\{0, \pm1\}$-ray with a negative cost \eqref{objective}, such that
\begin{enumerate}
    \item for every $i \in \llbracket k \rrbracket$ and $j \in \llbracket n \rrbracket$, at least one of $\bar{\beta}^i_j$ and $\bar{\gamma}^i_j$ are equal to $0$,
    \item for every $j \in \llbracket n \rrbracket$ where $\bar{w}_j = 0$, $\bar{\beta}^i_j = \bar{\gamma}^i_j = \bar{\alpha}_j = 0$ $\forall i \in \llbracket k \rrbracket$,
    \item and at least one of the following is true
\begin{enumerate}
\centering
    \item $\bar{\theta}^i_1 = 0$ $\forall i \in \llbracket k \rrbracket$
    \item $\bar{\theta}^i_2 = 0$ $\forall i \in \llbracket k \rrbracket.$
\end{enumerate}
\end{enumerate}
\end{proposition}
\proof{Proof.}
We first show that there exists a ray of $\bm{R}$ that satisfies property 1.
Since \eqref{main dual} is unbounded, let $(\bar{\beta}^1, \bar{\gamma}^1, \bar{\theta}^1, \dots, \bar{\beta}^k, \bar{\gamma}^k, \Tilde{\theta}^k, \bar{\alpha})$ be a $\{0, \pm1\}$-ray with negative objective cost \eqref{objective}. 
Suppose that there exists $i' \in \llbracket k \rrbracket$ and $j' \in \llbracket n \rrbracket$ such that $\bar{\beta}^{i'}_{j'}$ and $\bar{\gamma}^{i'}_{j'}$ are both equal to $1$, we consider the feasible solution where every other variables remain the same except $\bar{\beta}^{i'}_{j'}$ and $\bar{\gamma}^{i'}_{j'}$ are set to $0$. 
The new solution has an objective value smaller than the previous one by $s\hat{z}_{i'}(u_{j'}|w_{j'}| - l_{j'}|w_{j'}|) \geq 0$.
Since, by our assumption $(\bar{\beta}^1, \bar{\gamma}^1, \bar{\theta}^1, \dots, \bar{\beta}^k, \bar{\gamma}^k, \Tilde{\theta}^k, \bar{\alpha})$ has a negative cost, the new solution also has a negative cost.
Hence, by iterating this procedure until there is no pair $\bar{\beta}^i_j$ and $\bar{\gamma}^i_j$ both equal to $1$, we obtain a ray with property 1.\\

\noindent
Next, we prove that among the rays satisfying property 1, there exists at least one ray that satisfies property 2.
Let $(\bar{\beta}^1, \bar{\gamma}^1, \bar{\theta}^1, \dots, \bar{\beta}^k, \bar{\gamma}^k, \bar{\theta}^k, \bar{\alpha})$ be a $\{0, \pm1\}$-ray with property 1 and, among such vectors, has the smallest objective cost.
For every $j \in \llbracket n \rrbracket$ such that $\bar{w}_j = 0$, all $k$ constraints that involves $\alpha_j$ is given as follows:
$$\bar{\beta}^i_j - \bar{\gamma}^i_j + \bar{\alpha}_j = 0, \forall i \in \llbracket k \rrbracket.$$
Next, we will consider every possible values of $\{\bar{\beta}^i_j, \bar{\gamma}^i_j\}$ and $\bar{\alpha}_j$, and evaluate how much these variables contribute the objective.
\begin{enumerate}
    \item If $\bar{\alpha}_j = 0$, then $\bar{\beta}^i_j = \bar{\gamma}^i_j = 0$. In this case, the variables $\{\bar{\beta}^i_j, \bar{\gamma}^i_j | i \in \llbracket k \rrbracket \}$ and $\bar{\alpha}_j$ contribute a value of $0$ to the objective.
    \item If $\bar{\alpha}_j = 1$, then $\bar{\gamma}^i_j = 1$ and $\bar{\beta}^i_j = 0$. In this case, the variables $\{\beta^i_j, \gamma^i_j | i \in \llbracket k \rrbracket \}$ contribute to the objective cost by $sx_j \geq 0$.
    \item If $\bar{\alpha}_j = -1$, then $\beta^i_j = 1$ and $\gamma^i_j = 0$. In this case, the variables $\{\bar{\beta}^i_j, \bar{\gamma}^i_j | i \in \llbracket k \rrbracket \}$ contribute
    $$s(\sum_i \hat{z}_i|w_j|u_j - x_j|w_j|) \geq s(\sum \hat{z}_i|w_j|x_j - x_j|w_j|) \geq 0.$$
\end{enumerate}
Thus, $\bar{\beta}^i_j = \bar{\gamma}^i_j = \bar{\alpha}_j = 0$. \\

\noindent
Finally, we prove that among the rays satisfying both properties 1 and 2, there exists at least a ray that satisfies property 3.
Similarly, let $(\bar{\beta}^1, \bar{\gamma}^1, \bar{\theta}^1, \dots, \bar{\beta}^k, \bar{\gamma}^k, \bar{\theta}^k, \bar{\alpha})$ be a $\{0, \pm1\}$-ray with negative objective cost \eqref{objective} satisfying condition stated in property 1 and 2.
By contradiction and without loss of generality, suppose that $\theta^1_1 = 1$ and $\theta^1_2 = 1$.
For any $j \in \llbracket n \rrbracket$ and $\bar{w}_j \in \{-1, 1\}$, we have
$$\bar{\beta}^1_j - \bar{\gamma}^1_j + \bar{\theta}^1_1\bar{w}_j + \bar{\alpha}_j = 0$$
$$\bar{\beta}^2_j - \bar{\gamma}^2_j - \bar{\theta}^2_2\bar{w}_j + \bar{\alpha}_j = 0.$$
We can observe that if $\bar{\alpha}_j \neq 0$ then either $|\bar{\theta}^1_1\bar{w}_j + \bar{\alpha}|$ or $|\bar{\theta}^2_2w_j - \bar{\alpha}_j|$ is equal to 2, which is impossible because $|\bar{\beta}^i_j - \bar{\gamma}^i_j| \leq 1$ for every $i \in \llbracket k \rrbracket$, and $j \in \llbracket n \rrbracket$.
Hence, for such a $j \in \llbracket n \rrbracket$ and $\bar{w}_j \neq 0$, $\bar{\alpha}_j = 0$. 
Furthermore, from property 2, if $w_j = 0$ then $\bar{\alpha}_j = 0$. 
Thus, we conclude that $\bar{\alpha} = 0$, which contradicts our assumption that every ray with $\bar{\alpha} = 0$ has non-negative objective cost.
Therefore, at least one of the following must be true: \emph{(a)} $\bar{\theta}^i_1 = 0$ $\forall i \in \llbracket k \rrbracket$, \emph{(b)} $\bar{\theta}^i_2 = 0$ $\forall i \in \llbracket k \rrbracket.$
\endproof

\noindent
From Proposition \ref{gamma beta theta}, we can find a $\{0, \pm1\}$-ray with smallest objective cost \eqref{objective} by first setting $\bar{\beta}^i_j = \bar{\gamma}^i_j = \bar{\alpha}_j = 0$ for every $j \in \llbracket n \rrbracket$ where $\bar{w}_j = 0$. 
Next, since it is either $\theta^i_1 = 0 \ \forall i \in \llbracket k \rrbracket$ or $\theta^i_2 = 0 \ \forall i \in \llbracket k \rrbracket$, we will find the optimal solution for each cases.
We will provide a method to find a negative cost $\{0, \pm1\}$ ray when $\theta^i_2 = 0 \ \forall i \in \llbracket k \rrbracket$, and the case where $\theta^i_1 = 0 \ \forall i \in \llbracket k \rrbracket$ can be solved similarly. \\

\noindent
Let $K_0 = \{i \in \llbracket k \rrbracket | \bar{\theta}^i_1 = 0\}$ and $K_1 = \{i \in \llbracket k \rrbracket | \bar{\theta}^i_1 = 1\}$.
For any $j \in \llbracket n \rrbracket$, we consider two cases:
\begin{enumerate}
    \item If $\bar{w}_j = 1$, all $k$ constraints that involves $\alpha_j$ are of the forms:
    \begin{equation}
    \centering
        \begin{split}
            \bar{\beta}^i_j - \bar{\gamma}^i_j + \bar{\alpha}_j + 1 &= 0, \ \forall i \in K_1 \\
            \bar{\beta}^i_j - \bar{\gamma}^i_j + \bar{\alpha}_j &= 0, \ \forall i \in K_0.
        \end{split}
    \end{equation}
    
\noindent
Because $\bar{\alpha}_j$ can only take within a finite number of values, we can plug in every possible value of $\bar{\alpha}_j$, and pick the one that has the smaller cost. \\

If $\bar{\alpha}_j = 0$, then $\bar{\gamma}^i_j = 1, \ i \in K_1$, $\bar{\gamma}^i_j = 0, \ i \in K_0$ and $\bar{\beta}^i_j = 0, \forall i$, the cost that $\bar{\alpha}_j, \bar{\beta}^i_j$ and $\bar{\gamma}^i_j$ contributes to the objective function is $-s\sum_{i \in K_1} \hat{z}_i l_j |w_j|$.\\
If $\bar{\alpha}_j = -1$, then $\beta^i_j = 1, \ i \in K_0$, $\beta^i_j = 0, \ i \in K_1$ and $\gamma^i_j = 0, \forall i$, the cost coming from the $\bar{\alpha}_j, \bar{\beta}^i_j$ and $\bar{\gamma}^i_j$ is $s(\sum_{i\in K_0}\hat{z}_iu_j|w_j| - \hat{x}_j|w_j|)$. \\

\item If $\bar{w}_j = -1$, all $k$ constraints that involves $\bar{\alpha}_j$ are reduced to
\begin{equation}
\centering
    \begin{split}
        \bar{\beta}^i_j - \bar{\gamma}^i_j + \bar{\alpha}_j - 1 &= 0, \ \forall i \in K_1 \\
        \bar{\beta}^i_j - \bar{\gamma}^i_j + \bar{\alpha}_j &= 0, \ \forall i \in K_0.
    \end{split}
\end{equation}
If $\bar{\alpha}_j = 0$, then $\bar{\beta}^i_j = 1, \ i \in K_1$, $\bar{\beta}^i_j = 0, \ i \in K_0$ and $\gamma^i_j = 0, \forall i$, then similarly, the contribution cost of $\bar{\alpha}_j, \bar{\beta}^i_j$ and $\bar{\gamma}^i_j$ is $s\sum_{i \in K_1}\hat{z}_iu_j|w_j|$.\\
If $\bar{\alpha}_j = 1$, then $\bar{\gamma}^i_j = 1, \ i \in K_0$, $\bar{\gamma}^i_j = 0, \ i \in K_1$ and $\bar{\beta}^i_j = 0, \forall i$, then with the same manner, the contribution to the objective cost is $s(-\sum_{i\in K_0}\hat{z}_il_j|w_j| + \hat{x}_j|w_j|)$. \\
\end{enumerate}
\noindent
Motivated by the above observations, we define a set function $\psi$ on $2^{\llbracket k \rrbracket}$ that maps a set $K \in 2^{\llbracket k \rrbracket}$ to the minimum value of \eqref{main dual} when $\theta^i_1 = 1$ for every $i \in K$ and $\theta^i_1 = 0$ otherwise.
Let $\psi(K): 2^{\llbracket k \rrbracket} \rightarrow \mathds{R}$ be defined as
$$\psi(K) =  \sum_{i \in K}(h_i-b)\hat{z}_i + \sum_{j: w_j = 1} \min \{-\sum_{i \in K}\hat{z}_il_j|w_j|, \sum_{i \notin K}\hat{z}_iu_j|w_j| - \hat{x}_j|w_j| \} $$
$$ + \sum_{j: w_j = -1} \min \{-\sum_{i \notin K}\hat{z}_il_j|w_j| + \hat{x}_j|w_j|, \sum_{i \in K}\hat{z}_iu_j|w_j| \}.$$
To simplify the expression of the function $\psi$, we define
$$M^1_j \coloneqq u_j |w_j|, \text{ and }M^2_j \coloneqq l_j |w_j|.$$
Since $\sum_{i \in \llbracket k \rrbracket} \hat{\hat{z}}_i = 1$, we have
$$\min \{-\sum_{i \in K}\hat{z}_il_j|w_j|, \sum_{i \notin K}\hat{z}_iu_j|w_j| - \hat{x}_j|w_j| \} = \min\{M^1_j - \hat{x}_j|w_j|, \sum_{i \in K} \hat{z}_i (u_j - l_j)|w_j|\} - \sum_{i \in K}\hat{z}_iu_j|w_j|$$
$$\min \{-\sum_{i \notin K}\hat{z}_il_j|w_j| + \hat{x}_j|w_j|, \sum_{i \in K}\hat{z}_iu_j|w_j| \} = \min\{\hat{x}_j|w_j| - M^2_j, \sum_{i \in K} \hat{z}_i(u_j - l_j)|w_j|\} + \sum_{i \in K} \hat{z}_i l_j |w_j|,$$
If we plug these equations into $\psi(K)$ and define
$$\bar{h}_i \coloneqq h_i - b - \sum_{j: \bar{w}_j = 1}u_j|w_j| + \sum_{j:\bar{w}_j = -1} l_j|w_j|,$$
$$\Delta_j \coloneqq (u_j - l_j)|w_j| \geq 0,  \text{ and}$$
$$\bar{x}_j \coloneqq M^1_j - \hat{x}_j|w_j|, \text{ if } \bar{w}_j = 1, \bar{x}_j \coloneqq \hat{x}_j|w_j| - M^2_j \text{ otherwise},$$
then $\psi(K)$ can be written simply as
$$\psi(K) = \sum_{i \in K} \hat{z}_i\bar{h}_i + \sum_{j = 1}^n \min \{\sum_{i \in K}\hat{z}_i\Delta_j, \bar{x}_j\}.$$
We can see that finding an extreme ray of \eqref{recession cone} with a negative cost is now equivalent to finding a subset $K \in \llbracket k \rrbracket$ that minimize $\psi(K)$. In order to do so, we consider a continuous extension $\psi_c: [0,1]^k \rightarrow \mathds{R}$ of $\psi(K)$ defined as
$$\psi_c(q) = \sum_{i = 1}^k\bar{h}_i \hat{z}_i q_i + \sum_{j=1}^n \min \{\sum_{i = 1}^k\hat{z}_i\Delta_jq_i, \bar{x}_j\}.$$
Trivially $\psi_c(\mathds{1}_K) = f(K)$ for any $K \subseteq \llbracket k \rrbracket$, and $\psi_c(q)$ is a concave function because taking minimum preserves concavity and the summation of concave functions is concave. Next, without loss of generality, we assume that $\frac{\bar{x}_1}{\Delta_1} \leq \frac{\bar{x}_2}{\Delta_2} \leq \dots \leq \frac{\bar{x}_n}{\Delta_n}$, then  $\psi_c(s)$ is a piecewise linear function with $n+1$ pieces:
\[
    \psi_c(q)= 
\begin{cases}
    \sum_{i = 1}^k\bar{h}_i \hat{z}_i q_i + \sum_{j=1}^n \sum_{i = 1}^k\hat{z}_i \Delta_j q_i ,& \text{if } \sum_{i = 1}^k\hat{z}_iq_i \leq \frac{\bar{x}_1}{\Delta_1}\\
    \sum_{i = 1}^k\bar{h}_i \hat{z}_i q_i + \sum_{j=2}^n \sum_{i = 1}^k\hat{z}_i \Delta_j q_i + \bar{x}_1,& \text{if } \frac{\bar{x}_1}{\Delta_1} \leq \sum_{i = 1}^k\hat{z}_iq_i \leq \frac{\bar{x}_2}{\Delta_2}\\
    \ \ \ \vdots \\
    \sum_{i = 1}^k\bar{h}_i \hat{z}_i q_i + \sum_{i = 1}^k\hat{z}_i\Delta_nq_i + \sum_{j=1}^{n-1}\bar{x}_j,& \text{if } \frac{\bar{x}_{n-1}}{\Delta_{n-1}} \leq \sum_{i = 1}^k\hat{z}_iq_i \leq \frac{\bar{x}_n}{\Delta_n} \\
    \sum_{i = 1}^k\bar{h}_i \hat{z}_i q_i + \sum_{j=1}^n \bar{x}_j,& \text{if } \frac{\bar{x}_n}{\Delta_n} \leq \sum_{i = 1}^k\hat{z}_iq_i.
\end{cases}
\]
Because $\psi_c(x)$ is only defined on the box domain $[0,1]^n$, the minimum of $\psi_c(q)$ can be obtained by solving $n+1$ linear programs whose feasible domains are a box domain with a ranged constraint. 
Fortunately, these $n+1$ linear programming problems we need to solve are indeed relaxed knapsack problems (with some cost can be negative). 
Since they have similar structures, the solution of one problem can be used to derive the next problem as, shown in Algorithm \ref{algo:minimum-psi}.
In Algorithm \ref{algo:minimum-psi}, $\bar{H}$ is the matrix that stores the coefficient of every pieces of the function $\psi_c$, i.e., $\bar{H}_{i, v} = \hat{z}_i(\bar{h}_i + \sum_{j = v}^n \Delta_j)$.
For correctness and property of Algorithm \ref{algo:minimum-psi}, see Appendix \ref{app:algo1}.\\

\begin{algorithm}
\caption{Minimize $\psi_c(q)$}
\label{algo:minimum-psi}
\begin{algorithmic}[1]
\Procedure{Minimum$\Psi$}{$\bar{x}$, $\Delta$, $\hat{z}$, $\bar{H}$}
    \State $i \leftarrow 2, \ j \leftarrow 1$, $b \leftarrow \bar{x}/\Delta$ \Comment{breakpoints}
    \State $q^* \leftarrow 0^n$, $\psi^* \leftarrow 0$ \Comment{optimal solution \& optimal value}
    \State $s \leftarrow 0$ \Comment{weighted sum}
    \While{$\psi^* > 0$  and $i \leq n + 2$}
        \For{$v \in \{j, j+1, \dots, k\}$}
            \If{$\hat{z}_v = 0$}
                \State $j \leftarrow j+1$;
                \State continue;
            \EndIf
            \If{$\bar{H}_{i-1, v} \leq 0$}
                \If{$\hat{z}_v(1 - q^*_v) \leq b_i - s$}
                    \State $\psi^* \leftarrow \psi^* + \bar{H}_{i-1, v}(1 - q^*_v)\hat{z}_v$, $s \leftarrow s + \hat{z}_v(1-q^*_v)$
                    \State $q^*_v \leftarrow 1$, $j \leftarrow j +1$
                \Else
                    \State $q^*_v \leftarrow q^*_v + (b_i - s)/\hat{z}_v$
                    \State $\psi^* \leftarrow \psi^* + \bar{H}_{i-1, v}(b_i - s)$, $s \leftarrow b_i$
                \EndIf
            \Else \Comment{$\bar{H}_{i-1, v} \geq 0$}
                \If{$s \geq b_{i-1}$}
                    \State break
                \EndIf
                \If{$\hat{z}_v(1-q^*_v) \leq b_{i-1} - s$}
                    \State $\psi^* \leftarrow \psi^* + \bar{h}_{i-1,v}(1-q^*_v)\hat{z}_v$, $s \leftarrow s + \hat{z}_v(1 - q^*_v)$
                    \State $q^*_v \leftarrow 1$, $j \leftarrow j+1$
                \Else
                    \State $q^*_v \leftarrow q^*_v + (b_{i-1} - s)/\hat{z}_v$
                    \State $\psi^* \leftarrow \psi^* + \bar{H}_{i-1,v}(b_{i-1} - s)$, $s \leftarrow b_{i-1}$
                \EndIf
            \EndIf
        \EndFor
        \State $i \leftarrow i+1$
    \EndWhile
    \State{\Return $q^*, \psi^*$};
\EndProcedure
\end{algorithmic}
\end{algorithm}

\noindent
We can observe from Algorithm \ref{algo:minimum-psi} that we only have to use at most $O(\max(k, n))$ operations.
In addition, the optimal solution obtained from solving these knapsack problems will have at most one component being fractional; others are either $0$ or $1$. 
The next proposition draws the connection between the minimum value of $\psi(K)$ ($K \subseteq \llbracket k \rrbracket$) and the minimum value of its continuous relaxation $\psi_c(q)$ ($q \in [0, 1]^k$).

\begin{proposition}
\label{concave} Let $q^* \in [0,1]^k$ such that $q^*$ has only one non-binary entry and $\psi_c(q^*) < 0$. Then there exists $K \subseteq \llbracket k \rrbracket$ such that $\psi_c(\mathds{1}_K) < 0$.
\end{proposition}
\proof{Proof.}
Let the non-binary value of the vector $q^*$ be $\mu$ where $0 < \mu < 1$. We choose a vector $q_1$ such that $q_1$ is equal to $q^*$, except at the non-binary entry, we set it equal to $0$, and a vector $q_2$ such that $q_2$ is equal to $q^*$ except at the non-binary entry, we set it equal to $1$. We derive that $q^* = \mu q_2 + (1-\mu)q_1$. Since $\psi_c(q)$ is a concave function,  we have
$$0 > \psi_c(q^*) \geq \mu \psi_c(q_2) + (1-\mu)\psi_c(q_1).$$
Thus at least one of $\psi(q_1)$ or $\psi(q_2)$ has to be negative. \\
\endproof
\noindent
From Proposition \ref{concave}, supposing $q^* \in [0,1]^k$ is minimizer of $\psi_c(q)$, we can conclude that if $\psi_c(q^*) < 0$ then $\min_{K \subseteq \llbracket k \rrbracket} \psi(K) < 0$, and consequently, if $\psi_c(q^*) = 0$ then $\min_{K \subseteq \llbracket k \rrbracket} \psi(K) = 0$. 
Since the three properties stated in Proposition \ref{gamma beta theta} requires that every ray with $\bar{\alpha} = 0$ has non-negative cost, to completely answer if \eqref{main dual} has unbounded value, we still need to find the minimum of \eqref{main dual} when $\bar{\alpha} = 0$.
However, one can observe that in the separation scheme, we can start with $\Lambda = \{0^n\}$ in \eqref{lambda}, then for every $(\hat{x}, \hat{y}, \hat{z})$ satisfies \eqref{lambda} when $\alpha = 0^n$, every feasible solution of \eqref{main dual} with $\bar{\alpha} = 0^n$ has non-negative objective cost.

\begin{lemma}
The time complexity for the separation procedure, i.e., solving \eqref{main dual}, is $O(n\text{log}(n) + \max(k, n))$.
\end{lemma}

\proof{Proof.}
Because in Algorithm \ref{algo:minimum-psi}, we need to sort $\bar{x}$ and solve $n+1$ relaxed knapsacks, the number of operations is $O(n\text{log}(n) + \max(k, n))$.
\endproof

\subsection{Finite Optimal Solution}
In the previous subsection, we have described how to answer if \eqref{main dual} is unbounded.
In cases it is bounded, we need to find the optimal solution.
Our approach for finding the optimal extreme point is analogous to finding an unbounded extreme ray. 
Similarly, we are going to characterize the property of extreme points of $\bm{P}$ in a familiar manner as we did in Proposition \ref{gamma beta theta}.
\begin{proposition}
\label{gamma beta theta 2}
If \eqref{main dual} has finite negative optimal value and if every feasible solution where $\bar{\alpha} = \bar{w}$ has non-negative objective cost, then there exists a $\{0, \pm1\}$-solution with a negative cost \eqref{objective}, such that
\begin{enumerate}
    \item for every $i \in \llbracket k \rrbracket$ and $j \in \llbracket n \rrbracket$, at least one of $\bar{\beta}^i_j$ and $\bar{\gamma}^i_j$ is equal to $0$,
    \item for every $j \in \llbracket n \rrbracket$ where $\bar{w}_j = 0$, $\bar{\beta}^i_j = \bar{\gamma}^i_j = 0$ $\forall i \in \llbracket k \rrbracket$,
    \item and at least one of the following is true
    \begin{enumerate}
        \item $\bar{\theta}^i_1 = 0$ $\forall i \in \llbracket k \rrbracket$
        \item $\bar{\theta}^i_2 = 0$ $\forall i \in \llbracket k \rrbracket$.
    \end{enumerate}
\end{enumerate}
\end{proposition}

\noindent
Because of the similarity with Proposition \ref{gamma beta theta}, we omit the proof for brevity.
It is worth repeating that every property listed in Proposition \ref{gamma beta theta 2} is only true when every feasible solution whose $\bar{\alpha} = \bar{w}$ has non-negative objective cost. 
However, similar to the case with unboundedness, we start with an additional constraint corresponding with $\bar{\alpha} = \bar{w}$, i.e., initially $\Lambda = \{0^n, \bar{w}\}$. 
By Proposition \ref{gamma beta theta 2}, we can proceed with finding the optimal solution of \eqref{main dual} by a similar approach to the one presented in the previous subsection.
Since, at least one of $\bar{\theta}^i_1 = 0$ $\forall i \in \llbracket k \rrbracket$ or $\bar{\theta}^i_2 = 0$ $\forall i \in \llbracket k \rrbracket$ must be true, again we only need to focus on finding an optimal solution of $\eqref{main dual}$ when $\bar{\theta}^i_2 = 0$ $\forall i \in \llbracket k \rrbracket$.
Similarly, we define $K_0$ and $K_1$ as in the previous subsection. 
Since the right hand side of the constraint is not the $\bm{0}$ vector as in the case for extreme rays anymore, we denote $A_1 \coloneqq \{i \in \llbracket k \rrbracket| \ a_i = s\}$ and $A_0 \coloneqq \{i \in \llbracket k \rrbracket | a_i = 0\}$. 
All $k$ constraints that involve $\bar{\alpha}_j$ can be decoupled into $4$ groups:
\begin{equation}
    \notag
    \begin{split}
        \bar{\beta}^i_j - \bar{\gamma}^i_j + \bar{\alpha}_j = -\bar{w}_j, & \ \forall i\in K_1 \cap A_0 \\
        \bar{\beta}^i_j - \bar{\gamma}^i_j + \bar{\alpha}_j = 0, & \ \forall i \in K_1 \cap A_1 \\
        \bar{\beta}^i_j - \bar{\gamma}^i_j + \bar{\alpha}_j = 0, & \ \forall i \in K_0 \cap A_0 \\
        \bar{\beta}^i_j - \bar{\gamma}^i_j + \bar{\alpha}_j = \bar{w}_j, & \ \forall i \in K_0 \cap A_1.
    \end{split}
\end{equation}
Generally, we can say that a choice of $K_0$ and $K_1$ along with $A_0$ and $A_1$ defines a solution for \eqref{main dual}.
If both $K_1 \cap A_0$ and $K_0 \cap A_1$ are non-empty, since we are only looking for $\{0, \pm1\}$-solution, suppose there exists $j' \in \llbracket n \rrbracket$ such that $\bar{w}_j = \bar{\alpha}_{j'} = 1$, then we must have $\bar{\beta}^i_j - \bar{\gamma}^i_j = -2, \ \forall i \in K_1 \cap A_0$, which is impossible because $\bar{\beta}^i_j, \bar{\gamma}^i_j \in \{0,1\}$. 
Similar argument can be made for the case where $\bar{w}_{j'} = -1$ and $\bar{\alpha}_{j'} = -1$. 
Thus, if both $K_1 \cap A_0$ and $K_0 \cap A_1$ are not empty, $\bar{\alpha}$ must be $0^n$. Secondly, if $K_1 = A_1$ and $K_0 = A_0$, then the optimal solution in this case will be $\bar{\beta}^i_j = \bar{\gamma}^i_j = \bar{\alpha}_j = 0$.
Finally, we are left with two cases, either $A_1 \subset K_1$ or $K_1 \subset A_1$.\\

\noindent
If $A_1 \subset K_1$, and let $K \coloneqq K_1 \setminus A_1$, then all $k$ constraints involving $\bar{\alpha}_j$ are given as
\begin{equation}
    \notag
    \begin{split}
        \bar{\beta}^i_j - \bar{\gamma}^i_j + \bar{\alpha}_j = -\bar{w}_j & \ \forall i \in K \\
        \bar{\beta}^i_j - \bar{\gamma}^i_j + \bar{\alpha}_j = 0 & \ \forall i \in \llbracket k \rrbracket \setminus K.
    \end{split}
\end{equation}
With this, we can define $\psi(K)$ exactly the same as before. The only difference here is that now we minimize the set function $\psi(K)$ with one additional constraint, which is $K \subset A_0$. For the case where $K_1 \subset A_1$, we define $K \coloneqq A_1 \setminus K_1$. We have all $k$ constraints that involve $\bar{\alpha}_j$ as follows
\begin{equation}
    \notag
    \begin{split}
        \bar{\beta}^i_j - \bar{\gamma}^i_j + \bar{\alpha}_j = \bar{w}_j & \ \forall i\in K \\
        \bar{\beta}^i_j - \bar{\gamma}^i_j + \bar{\alpha}_j = 0 & \ \forall i \in \llbracket k \rrbracket \setminus K.
    \end{split}
\end{equation}
From this, we can define set function $\psi(K)$ exactly the same and minimize it over $K \subset A_1$.

\section{Separation for Piecewise Linear Activation Functions}
\label{General piecewise linear function}
So far, we have been constraining the univariate activation function $f$ to be a staircase function.
Now, we will relax this constraint and show how we can use results in previous sections as building blocks to derive a separation procedure for general piecewise linear activation functions.
Our approach is based on the following two propositions.
\begin{proposition}
\label{prop: staircase decomp}
Let $f: [L, U] \rightarrow \mathds{R}$ be a continuous univariate piecewise linear function with $k$ pieces. Then, there exists continuous staircase functions $f_1, \dots, f_m$ such that $f = f_1 + \dots + f_m$, where $m \leq k$. Furthermore, if $f$ is discontinuous, then there exist continuous staircase functions $f_1, \dots, f_m$ and one constant piecewise function $f_0$ such that $f = f_0 + f_1 + \dots f_m$
\end{proposition}
\proof{Proof.}
First, we give proof for the case where $f$ is continuous. Since $f$ is a piecewise linear function with $k$ pieces, let $h_0 \coloneqq L, h_1, \dots, h_{k-1}, h_k \coloneqq U$ be its breakpoints, and $a_1, \dots, a_k$ be the slopes for each pieces of $f$ respectively.
Moreover, we denote $\{s_1, \dots, s_m\}$, where $m \leq k$, as the set of slopes of all pieces in $f$.
Let $S_v = \{i \in \llbracket k \rrbracket | \ a_i = s_v\}$ for each $v \in \llbracket m \rrbracket$ and let $f_v$ be a continuous piecewise linear function such that $f_v(L) = \frac{f(L)}{m}$, $f_v(t)$ and $f(t)$ have the same slope if $t \in \cup_{i \in S_v} [h_{i-1}, h_{i}]$, and $f_v(t)$ is constant otherwise.
Because of the continuity of $f_v$, $f_v$ is uniquely defined.
We can see that, for every $v \in \llbracket m \rrbracket$,  every non-constant pieces of $f_v(t)$ has the same slope, thus $f_v(t)$ is a staircase function. 
In addition, by our choice of $f_1, \dots, f_m$, we have $f = f_1 + \dots + f_m$. \\

\noindent
For the case where $f$ is discontinuous with $k$ pieces, we can construct a continuous univariate function $f'$ such that the slopes of $f'$ agree with that of $f$ on all $k$ pieces.
Then, the difference between $f$ and $f'$ is only a constant piecewise function $f_0$.
By applying the result in the continuous case, we have that $f$ can be written as a sum of continuous staircase functions and one constant piecewise function.
\endproof

\noindent
By Proposition \ref{prop: staircase decomp}, we can decompose any univariate piecewise linear function into sum of $m$ staircase functions, where $m \leq k+1$.
However, in some cases, $m$ can be much smaller than $k$, as illustrated in the next example.
\begin{figure}[H]
\centering
\includegraphics[width=\textwidth]{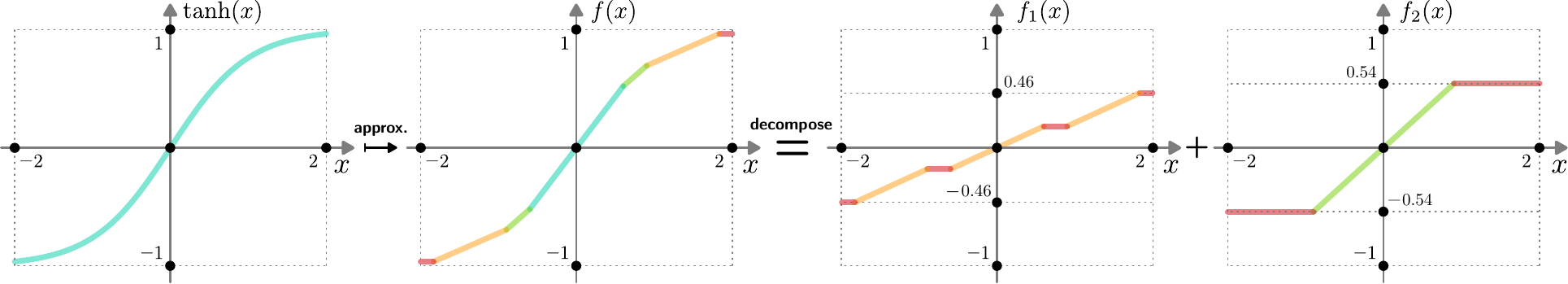}
\caption{As a motivation for applying our verification method for neural nets with any activation function, in this example, we approximate the tanh function by a piecewise linear function with $7$ pieces. Because of the symmetry of the tanh function and a ``good" way of choosing the slopes for each piece, we can decompose the linear approximation of tanh into just $2$ staircase functions.}
\label{fig:knapsacks2D}
\end{figure}
\noindent
As shown in the next lemma, the number of staircase functions plays an important role in the time complexity for the separation of general piecewise linear functions.
\begin{lemma}
\label{projection is cayley}
Let $g = g_1 + \dots + g_m$ where $g_v(x) = f_v(w \cdot x)$ and $f_v$ is a staircase function for every $v \in \llbracket m \rrbracket$. If $f_1, \dots, f_m$ have the same breakpoints, then the polyhedron $P$ is defined as follows
\begin{align}
\centering
\notag
(x, y_l, z) \in C(g_l), & \ \forall v \in \llbracket m \rrbracket \\
\notag
y = y_1 + \dots + y_m &,
\end{align}
has its projection onto the $(x, y, z)$-space equal to the convex hull of the Cayley embedding of the graph of $g$, i.e., 
$$\Pi_{x, y, z}(P) \coloneqq \{(x, y, z) | \ \exists y_1, \dots, y_m \text{ such that } \ (x, y_v, z) \in C(g_v) \ \forall v \in \llbracket m \rrbracket, \ y = y_1 + \dots + y_m  \} = C(g).$$
\end{lemma}
\proof{Proof. }
The fact that $C(g) \subseteq \Pi_{x,y,z}(P)$ follows from by the definition of projection of $P$.
To show that $\Pi_{x,y,z}(P) \subseteq C(g)$, we show that a point $(x,y,z) \in \Pi_{x,y,z}(P)$ also belongs to $C(g)$. 
Since $(x,y,z) \in \Pi_{x,y,z}(P)$, there exists $y_1, \dots, y_m$ such that $(x, y, y_1, \dots, y_m, z) \in P$. Since $P$ is a polyhedral, every point can be written as convex combination of its extreme points, and because every extreme point of $P$ must have the $z$ component being a unit vector, we have that there exists $\lambda \in [0,1]^k$, and $x_i \in D_i, \forall i \in \llbracket k \rrbracket$ such that
\begin{equation}
    \notag
    \begin{split}
        \sum_{i = 1}^k \lambda_i &= 1 \\
        \sum_{i=1}^k \lambda_i x_i &= x \\
        \sum_{i=1}^k \lambda_i g_v(x_i)&= y_v, \ \forall v \in \llbracket m \rrbracket. \\
    \end{split}
\end{equation}
Thus, $y = \sum_{v = 1}^m y_v = \sum_{v=1}^m \sum_{i=1}^k \lambda_i g_v(x_i) = \sum_{i=1}^k \lambda_i g(x_i)$. Hence, $(x,y,z) \in C(g)$.
\endproof
\noindent
From \eqref{original cayley upper} and \eqref{original cayley lower}, we can see that each $\bar{\alpha} \in \mathds{R}^n$ or $\underline{\alpha} \in \mathds{R}$ gives a valid inequality for the convex hull of Cayley embedding.
Since for every $v \in \llbracket m \rrbracket$, conv$(C(g_v))$ is a polyhedron, there must exist a finite set of inequalities that defines conv$(C(g_v))$. 
We denote $\bar{\Lambda}_v \subset \mathds{R}^n$ and $\underline{\Lambda}_v \subset \mathds{R}^n$ as finite subsets of $\bar{\alpha}$ and $\underline{\alpha}$ corresponding to the facet defining upper bound and lower bound of $y_v$ respectively.
With this notation, conv$(C(g_l))$ is given as follows:
\begin{subequations}
\label{cayley gl}
\begin{align} 
        \label{original cayley upper gl}
        y_l \leq \alpha_l \cdot x & + \sum_{i=1}^k (\max_{x^i \in D^i} (a^l_iw- \alpha_l) \cdot x^i +b_i)z_i,  \ \forall \bar{\alpha}_l \in \bar{\Lambda}_l \\
        \label{original cayley lower gl}
        y_l \geq \underline{\alpha}_l \cdot x &  + \sum_{i=1}^k (\min_{x^i \in D^i} (a^l_iw- \underline{\alpha}_l) \cdot x^i +b_i)z_i,  \ \forall \underline{\alpha}_l \in \underline{\Lambda}_l \\
        & (x, y_l, z) \in D \times \mathds{R} \times \Delta^k.
\end{align}
\end{subequations}
Because of Proposition \ref{projection is cayley}, to derive the set of inequalities describing $C(g)$, we can perform the Fourier-Motzkin projection of $P$ onto the $(x,y,z)$-space and derive Proposition \ref{cayley embedding}.
\begin{proposition}
\label{cayley embedding}
Let $g = g_1 + \dots + g_m$ where $g_1, \dots, g_k$ are staircase functions, then conv$(C(g))$ is the solutions of the following system
\begin{subequations}
\label{eq:cayley_general}
\begin{align}
    \notag
    y \leq \alpha_1 \cdot x & + \sum_{i=1}^k (\max_{x^i \in D^i} (a^1_iw- \alpha_1) \cdot x^i +b_i)z_i + \dots + \\
     \alpha_m \cdot x & + \sum_{i=1}^k (\max_{x^i \in D^i} (a^m_iw- \alpha_m) \cdot x^i +b_i)z_i, \ \forall \bar{\alpha}_1 \in \bar{\Lambda}_1, \dots, \bar{\alpha}_m \in \bar{\Lambda}_m \\
    \notag
    y \geq \underline{\alpha}_1 \cdot x &  + \sum_{i=1}^k (\min_{x^i \in D^i} (a^1_iw- \underline{\alpha}_1) \cdot x^i +b_i)z_i + \dots + \\
    \underline{\alpha}_m \cdot x &  + \sum_{i=1}^k (\min_{x^i \in D^i} (a^m_iw- \underline{\alpha}_m) \cdot x^i +b_i)z_i, \ \forall \underline{\alpha}_1 \in \underline{\Lambda}_1, \dots, \underline{\alpha}_m \in \underline{\Lambda}_m\\
    & (x, y, z) \in D \times \mathds{R} \times \Delta^k.
\end{align}
\end{subequations}

\end{proposition}
\noindent
We can observe that the right-hand side of \eqref{eq:cayley_general} consists of $m$ terms, where each term corresponds to a minimization problem. For each $v \in \llbracket m \rrbracket$, solving the $v^{\text{th}}$ term is equivalent to perform the separation procedure described in Section \ref{sep procedure}. Hence, the separation procedure for general piecewise linear functions can be done by performing $m$ separations of staircase functions.

\section{Experimental Results}
\label{sec:experiments}
In this section, we perform the verification task on quantized neural networks.
We choose this type of network to perform the experiments because different quantization levels correspond to a different number of pieces in the piecewise linear activation function.
Thus, it allows us to test our techniques on activation functions with many pieces.
We compare our formulation for both exact and inexact verifiers with a Big-M formulation.
As we expect from the strength of an ideal formulation, the Cayley embedding formulation with a fast separation procedure always gives a tighter bound for inexact verifiers and runs faster than existing methods in most instances.\\

\noindent
In the following experiments, we conduct the verification for fully-connected neural networks with two hidden layers and 256 neurons for each layer on the MNIST dataset \citep{deng2012mnist}.
The only difference between the network's architecture is their activation function.
To this end, we choose the Dorefa functions \citep{zhou2016dorefa} for the activation functions, as it is a common activation in quantized neural networks.\\

\noindent
\subsection{Inexact Verifiers}
For inexact verifiers, we compare our method with the Big-M formulation and DeepPoly.
To our knowledge, no recorded propagation-based method was developed for quantized neural network verification, and DeepPoly was initially developed for ReLU activation function.
However, in our experiments, we extend the DeepPoly method \citep{singh2019abstract} to work on quantized networks based on its main idea of using two linear constraints to approximate the graph of the activation function.
For both the Big-M formulation and the Cayley embedding formulation, we use the result from DeepPoly for the bound of the pre-activation value of a neuron. \\

\begin{table}[h]
\begin{center}
\caption{Inexact Verifiers.}
\label{tab:ineaxact}
\resizebox{\textwidth}{!}{
\begin{tabular}{c c rc rc rc}
\toprule
\multirow{2}{*}{NN} & \multirow{2}{*}{$\epsilon$} & \multicolumn{2}{c}{\textbf{DeepPoly}} & \multicolumn{2}{c}{\textbf{Big-M Formulation}} & \multicolumn{2}{c}{\textbf{Cayley Embedding}}\\
& & \#Verified & Time (s) & \#Verified & Time (s) & \#Verified & Time (s) \\
\midrule
\multirowcell{4}{Dense $2\times 256$\\ Dorefa $\kappa=2$} 
 & 0.008 & 118 & $0.338 \pm 0.056$ & \textbf{138} & $1.060 \pm 0.005$ & \textbf{138} & $1.100 \pm 0.008$ \\
 & 0.016 & 59  & $0.338 \pm 0.058$ & 112 & $1.056 \pm 0.006$ & \textbf{113} & $1.129 \pm 0.086$ \\
 & 0.024 & 19  & $0.336 \pm 0.055$ & 65  & $1.075 \pm 0.004$ & \textbf{66 } & $1.139 \pm 0.078$ \\
 & 0.032 & 0   & $0.326 \pm 0.054$ & 28  & $1.080 \pm 0.006$ & \textbf{29 } & $1.174 \pm 0.086$ \\
\midrule
\multirowcell{4}{Dense $2\times 256$\\ Dorefa $\kappa=3$} 
 & 0.008 & 132 & $0.339 \pm 0.059$ & \textbf{142} & $1.056 \pm 0.005$ & \textbf{142} & $1.102 \pm 0.075$ \\
 & 0.016 & 87  & $0.340 \pm 0.059$ & \textbf{125} & $1.058 \pm 0.005$ & \textbf{125} & $1.120 \pm 0.070$ \\
 & 0.024 & 11  & $0.341 \pm 0.058$ & 90  & $1.078 \pm 0.005$ & \textbf{91 } & $1.169 \pm 0.079$ \\
 & 0.032 & 0   & $0.324 \pm 0.052$ & 27  & $1.080 \pm 0.006$ & \textbf{29 } & $1.210 \pm 0.090$ \\
\midrule
\multirowcell{4}{Dense $2\times 256$\\ Dorefa $\kappa=4$} 
 & 0.008 & 132 & $0.329 \pm 0.055$ & 143 & $1.082 \pm 0.005$ & \textbf{144} & $1.113 \pm 0.082$ \\
 & 0.016 & 78  & $0.329 \pm 0.056$ & \textbf{126} & $1.063 \pm 0.006$ & \textbf{126} & $1.134 \pm 0.072$ \\
 & 0.024 & 6   & $0.330 \pm 0.056$ & 86  & $1.071 \pm 0.006$ & \textbf{90 } & $1.178 \pm 0.086$ \\
 & 0.032 & 0   & $0.331 \pm 0.056$ & 25  & $1.100 \pm 0.006$ & \textbf{34 } & $1.286 \pm 0.160$ \\
\midrule
\multirowcell{4}{Dense $2\times 256$\\ Dorefa $\kappa=5$} 
 & 0.008 & 140 & $0.329 \pm 0.056$ & \textbf{143} & $1.060 \pm 0.006$ & \textbf{143} & $1.130 \pm 0.083$ \\
 & 0.016 & 78  & $0.332 \pm 0.056$ & 138 & $1.087 \pm 0.005$ & \textbf{140} & $1.169 \pm 0.078$ \\
 & 0.024 & 4   & $0.331 \pm 0.056$ & 98  & $1.107 \pm 0.007$ & \textbf{100} & $1.256 \pm 0.113$ \\
 & 0.032 & 1   & $0.328 \pm 0.056$ & 33  & $1.144 \pm 0.007$ & \textbf{44 } & $1.409 \pm 0.190$ \\
\bottomrule
\end{tabular}
}
\end{center}
\end{table}
\noindent

\noindent
In the experiments summarized in Table \ref{tab:ineaxact}, we trained four different neural networks using Larq\citep{larq} and then performed the verification tasks on these networks with four different values of $\epsilon$. 
For each network and each perturbation $\epsilon$, we run three algorithms on the same dataset of 150 images and report the number of images verified to be robust within the $\epsilon$ perturbation. The time and the standard deviation are computed based on the run time when performing the verification task on a 150 images dataset, sampled randomly from MNIST images, which were not used during the neural network training phases. \\

\noindent
As expected, the Cayley embedding formulation provides the best lower bound for the number of robust images in every instance.
In particular, the gap between the bound produced Big-M formulations and the bound by Cayley embedding formulations is enlarged as $\epsilon$ increases and the number of pieces in the activation function increases from Dorefa $2$ (4 pieces) to Dorefa $5$ (32 pieces).
Certainly, the improvement in the lower bound compromises the verifying time.
With the state-of-the-art linear programming solver Gurobi \citep{gurobi} and our efficient separation procedure, the solve time for the Cayley embedding formulation takes roughly $1$ seconds longer than the propagation-based method like DeepPoly per image.
However, the Cayley embedding provides much better bounds, thus providing a compelling trade-off between solve times and the quality of the bounds.

\subsection{Exact Verifiers}
Next, we evaluate the performance of the Cayley embedding formulation as an exact verifier.
For each neural network, image, and perturbation $\epsilon$, we try to solve the associated MIP problems optimally.
We also perform the same experiment with Big-M formulation for comparison. \\

\begin{table}[]
\caption{Cayley Embedding as Exact Verifiers.}
\label{tab:cayley_exact}
\resizebox{\textwidth}{!}{
\begin{tabular}{llllrr}
\toprule
\multicolumn{1}{c}{\multirow{2}{*}{NN}} & \multicolumn{1}{c}{\multirow{2}{*}{$\epsilon$}} & \multicolumn{4}{c}{\textbf{Cayley Embedding Formulation}} \\
\multicolumn{1}{c}{} & \multicolumn{1}{c}{} & \#Nodes   & Gap (\%)   & Gurobi Time (s)  & User Callbacks (s)\\
\midrule
Dorefa 2 & \multirow{3}{*}{0.008}       & $2984.4 \pm 1590.1$       & $0.00$             & $2.84 \pm 0.66 $  & $1.14 \pm 0.4$  \\
Dorefa 3 &                              & $53277.0 \pm 18666.20$    & \textbf{4.19 $\pm$ 1.74}    & Timeout           & $17.73 \pm 5.12$  \\
Dorefa 4 &                              & $33248.4 \pm 268.06$      & \textbf{4.28 $\pm$ 1.06}    & Timeout           & $14.09 \pm 0.32$  \\
\midrule
Dorefa 2 & \multirow{3}{*}{0.016}       & $45925.4 \pm 17388.72$    & \textbf{11.57 $\pm$ 5.70}   & Timeout           & $16.51 \pm 6.52$  \\
Dorefa 3 &                              & $33406 \pm 639.79$        & \textbf{12.33 $\pm$ 6.09}   & Timeout           & $14.46 \pm 0.35$  \\
Dorefa 4 &                              & $42701.2 \pm 20587.1$     & \textbf{9.34 $\pm$ 6.22}    & Timeout           & $19.63 \pm 9.63$  \\
\bottomrule
\end{tabular}
}
\end{table}

\noindent
In experiments presented in Table \ref{tab:cayley_exact} and Table \ref{tab:bigm_exact}, for $\epsilon = 0.008$, we set a maximum solving time of 2 minutes, and for $\epsilon = 0.016$, the timeout is 3 minutes. The total time or verifying time contains the process of solving LPs (Gurobi Time) and the time spent on deriving separating hyperplanes (User Callbacks).
For each neural network and perturbation, we solve the MIP problems for verifying a set of randomly selected 10 images from MNIST. 
We report the nodes explored in the branch and bound trees. 
If the program hits the preset time limit, we record the gap between the current feasible solution and the best bound available. \\

\begin{table}[]
\centering
\caption{Big-M Formulation as Exact Verifiers.}
\label{tab:bigm_exact}
\begin{tabular}{lllrr}
\toprule
\multicolumn{1}{c}{\multirow{2}{*}{NN}} & \multicolumn{1}{c}{\multirow{2}{*}{$\epsilon$}} & \multicolumn{3}{c}{\textbf{Big-M Formulation}} \\
\multicolumn{1}{c}{} & \multicolumn{1}{c}{} & \#Nodes   & Gap (\%)   & Gurobi Time (s)  \\
\midrule
Dorefa 2 & \multirow{3}{*}{0.008}       & $3925.5 \pm 2326.01$    & $0.00$             & $2.84 \pm 0.66 $ \\
Dorefa 3 &                              & $51285.8 \pm 20756.89$  & $5.89 \pm 4.37$    & Timeout \\
Dorefa 4 &                              & $33063.8 \pm 607.23$    & $4.46 \pm 1.64$    & Timeout \\
\midrule
Dorefa 2 & \multirow{3}{*}{0.016}       & $33340.6 \pm 427.03$    & $13.09 \pm 4.90$   & Timeout \\
Dorefa 3 &                              & $33224.5 \pm 317.93$    & $12.48 \pm 5.08$   & Timeout \\
Dorefa 4 &                              & $33091.6 \pm 406.6$     & $11.41 \pm 7.90$   & Timeout \\
\bottomrule
\end{tabular}
\end{table}

\noindent
Similar to the experiments we perform on the Cayley embedding formulation, with Big-M formulation, we also set the same timeout for each value of $\epsilon$, and then record the number of nodes explored, solving time and gap for timeout cases.
In case we can solve the verification MIP problem within the time limit, Dorefa 2 and $\epsilon = 0.008$ in particular, the Cayley embedding formulation is slightly slower than Big-M formulation.
However, using the Cayley embedding, the number of nodes in the branch and bound trees are smaller in this case.
When the solver hits the time limit, in most cases, the Cayley embedding formulation either gives a better feasible solution or attains a smaller upper bound of the objective function, which results in a better gap.

\section{Conclusion}
In this work, we provide a strong formulation for a trained neural network by constructing an ideal MIP formulation for each neuron.  
We first derive a fast separation procedure for a strong formulation of neurons with staircase activation functions. 
Using the separation for staircase functions as a building block, we derive a separation method for neurons with general piecewise linear activation functions. 
We tested our formulations and separation procedure for verifying quantized neural networks. 
Empirically, we showed that our formulation is stronger than a Big-M formulation for piecewise linear function. 
In particular, when used as a relaxed verifier, our formulation always returns better objective value than the Big-M formulation or a propagation-based method like DeepPoly.
When used as an exact verifier, our formulation performs better than Big-M formulation in most instances in terms of speed (solve time), the ability to find a feasible solution and bounds for the branch-and-bound tree.

%
%
%
\begin{APPENDICES}
\section{Retrieving Separation Cut}

\label{appendix: separation cut}
This section provides a fast algorithm to retrieve the coefficient of the variables $z$ in \eqref{original cayley upper} or \eqref{original cayley lower} given a vector $\bar{\alpha} \in \mathds{R}^n$ or $\underline{\alpha} \in \mathds{R}^n$. 
In particular, we need to solve the following series of linear programs:
$$\max_{x^i \in D^i} (a_iw - \bar{\alpha})x^i \text{ or } \min_{x^i \in D^i} (a_iw - \underline{\alpha}) x^i \quad \forall i \in \llbracket k \rrbracket.$$
Since $a_i$ can only take values between $0$ or a positive real number $s$, $a_iw - \bar{\alpha}$ or $a_iw - \underline{\alpha}$ can only take between $2$ values.
Thus, we are in fact considering a set of LP problems of the form:
\begin{equation}
\label{eq:knapsackseries}
\max_{x^i \in S^i} c \cdot x^i \quad \forall i \in \llbracket k \rrbracket,
\end{equation}
where $S^i = \{x \in \mathds{R}^n| \ x \in [L, U], \ h_{i-1} \leq w \cdot x \leq h_i \}$ for every $i \in \llbracket k \rrbracket$.
Here, we further assume that $h_0 = \min_{x \in [L, U]} w \cdot ~x$ and $h_k = \max_{x \in [L, U]} w \cdot x$ to guarantee $S^i \neq \emptyset$ for every $i \in \llbracket k \rrbracket$ and $\cup_{i \in \llbracket   k \rrbracket} S^i = [L, U]$. \\

\noindent
The first step in solving all $k$ linear programs \eqref{eq:knapsackseries} is to observe that, since $\cup_{i \in \llbracket k \rrbracket} S^i = [L, U]$, the optimal solution $x^*$ of $\max_{x \in [L,U]} c \cdot x$ is also the optimal solution of $\max_{x^{i^*} \in S^{i^*}} c \cdot x^{i_0}$ for some $i^* \in \llbracket k \rrbracket$.
Next, we can see that the solution $\hat{x}^i$ of $\max_{x^i \in S^i} c \cdot x$ will be always active at the lower bound inequality, i.e., $w \cdot \hat{x}^i = h_{i_1}$ for $i > i^*$ and will be always be active at upper bound inequality, i.e., $w \cdot \hat{x}^i = h_i$ for $i < i^*$.
Based on this property, we can derive a simple algorithm to obtain all $\hat{x}^i$ from $x^*$.
\begin{figure}[H]
\centering
\includegraphics[width=50mm]{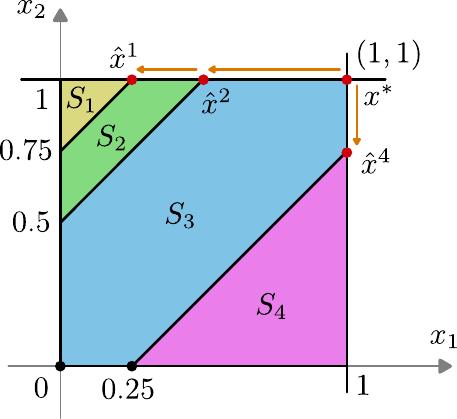}
\caption{\small An illustration of an algorithm for retrieving separation cut. In this example, we suppose that the objective vector $c$ in \eqref{eq:knapsackseries} is $(1, 1)$ and the unit box domain is partition into $\{S^1, S^2, S^3, S^4\}$ by parallel planes. We first find the optimal solution of $\max c \cdot x$ over the box domain, which is $x^* = (1,1)$ and $i^* = 3$. From then, we update component $x^*$ by order of increasing objective coefficient until the new solution meet the lower bound of $S^2$. This procedure is iterated until we get the solution for every $S^i$. }
\label{fig:knapsacks2D}
\end{figure}

\section{Property of Algorithm \ref{algo:minimum-psi}}
\label{app:algo1}
This section provides the proof for correctness of Algorithm \ref{algo:minimum-psi} as long as some of its properties.
It is worth to remind that Algorithm \ref{algo:minimum-psi} only finds the minimum of the continuous function $\psi_c(q)$ where the variable $\theta^i_2$ is set to be zero for every $i \in \llbracket k \rrbracket$.
In practice, we need to run a similar procedure to Algorithm \ref{algo:minimum-psi} to find the minimum of $\psi_c$ for the case $\theta^i_1 = 0, \forall \ i \in \llbracket k \rrbracket$.
The second procedure is indeed Algorithm \ref{algo:minimum-psi} when $\bar{w}$ is replaced by $-\bar{w}$, and therefore we omit the detail for brevity.
Hence, it is enough to only prove the correctness of Algorithm \ref{algo:minimum-psi} for the case $\theta^i_2 = 0, \forall \ i \in \llbracket k \rrbracket$.

\begin{lemma}
Given $\bar{x}, \Delta, z, \bar{h}$ as defined in section \ref{sep procedure}, and suppose that $\min_{q \in [0,1]^k} \psi_c(q) < 0$, Algorithm \ref{algo:minimum-psi} returns an $\{0, 1\}$-vector $q^*$ with negative objective $\psi^*$.
\end{lemma}
\begin{proof}
We will show that the algorithm find the minimum of the piecewise linear function $\psi_c(q)$ for each pieces before moving on the next.
Certainly, by line 9-11 of Algorithm \ref{algo:minimum-psi}, if for some $v \in \llbracket k \rrbracket$ and $z_v = 0$, then the coefficient of $q_v$ in $\psi_c$ is zero for every pieces.
In which case, the changing value of $q_v$ does not effect the value of the function $\psi_c$ nor the sum $\sum_{i = 1}^k z_i q_i$, which determines the pieces of $\psi_c$.
Thus, we can set $q_v = 0$ for every $v$ such that $z_v = 0$ without compromising the finding of an minimum value for $\psi_c$.\\

\noindent
When $i = 2$, the first while loop ensure the current $q$ is the minimizer of $\psi_c$ when $0 \leq \sum_{i=1}^k q_i z_i \leq b_1 \coloneqq \frac{\bar{x}_1}{\Delta_1}$ as it set $q_v = 0$ if $H_{1, v} > 0$ and set $q_v$ to be largest possible value without violating any constraints.
Now, suppose that during the algorithm, we obtain $q$ as the optimal solution of some piece $i$ with non-negative value of $\psi_c$ (otherwise the algorithm would stop), we will show that in the $i+1$ while loop, we obtain the optimal solution for the $i+1$ piece.\\

\noindent
If $\bar{H}_{i-1, v} < 0$ then $\bar{H}_{i, v} < 0$ for every $v \in \llbracket k \rrbracket$ as $\Delta_j > 0$ for every $j \in \llbracket n \rrbracket$.
Thus, if we already set $q_v = 1$ because of $\bar{H}_{i-1, v}$ is non-positive in the $i$ step, we also keep $q_v = 1$ in the $i+1$ step as in the $i+1$, not only $\bar{H}_{i, v}$ is smaller, but also the lower bound of $\sum_{i=1}^k z_i q_i$ increases.
As the value of $H_{i, v}$ is increasing in $v$, we only need to consider the index $v$ where $H_{i, v} > 0$.
In case $H_{i, v} > 0$, because we are minimizing, we want the value of $q_v$ is as small as possible but also satisfies the lower bound constraint.
Line 25-29 of Algorithm \ref{algo:minimum-psi} set the $q_v$ in the increasing order of $H_{i, v}$ until its meet the lower bound constraint of the $i+1$ piece.
Thus after the $i+1$ iteration, we obtain minimum of the $i+1$ piece of $\psi_c$.
By our assumption that $\min_{q \in [0,1]^k} \psi_c(q) < 0$, the algorithm will stop and return $q^*$ with negative objective value $\psi^*$.
\end{proof}

\noindent
Certainly, after $n+1$ iterations, if Algorithm \ref{algo:minimum-psi} return an positive value $\psi^*$, this is one proof that the solution $(x, y, z)$ is feasible.
One point worth mentioning is that, in practice, we do not always have $\frac{\bar{x}_1}{\Delta_1} \leq \frac{\bar{x}_2}{\Delta_2} \leq \dots \leq \frac{\bar{x}_n}{\Delta_n}$.
Hence, in implementing the algorithm, we need an additional step of sorting the vector $b \coloneqq \bar{x}/ \Delta$ defined in line 3 of the algorithm.
The mapping of the sorting $p$ is kept to recover the value for $\alpha$, as the non-zero index of $\alpha$ will be the first $i^*$ value of $p$ for $i^*$ is total number of while loop after finishing Algorithm \ref{algo:minimum-psi}.
Let $\text{supp}(\alpha) \coloneqq \{j \in \llbracket n \rrbracket | \alpha_j \neq 0\}$ be its support and $\|\text{supp}(\alpha)\|$ denotes its cardinality.
As we stop the Algorithm \ref{algo:minimum-psi} as soon as we find a piece of $\psi_c$ with negative value, the retrieved $\alpha$ has the minimal support, i.e., no other value of $\alpha'$ with negative cost such that $\text{supp}(\alpha') \subsetneq \text{supp}(\alpha)$
This property draw a connection between $\alpha$ and the extreme rays of $\mathbb{P}$.

\begin{proposition}
The value of $\alpha$ retrieved from $q^*$ by Algorithm \eqref{algo:minimum-psi}, along with optimal solution $\beta, \gamma, \theta$ of \eqref{main dual} is an extreme ray of $\mathbb{P}$ if $\alpha \neq \bar{w}$. 
\end{proposition}

\begin{proof}
By contradiction, suppose that the solution $p \coloneqq (\beta^1, \gamma^1, \theta^1, \dots, \beta^k, \gamma^k, \theta^k, \alpha)$ is not extreme point of $\mathbb{P}$.
Let $\bar{p} \coloneqq \bar{\beta^1}, \bar{\gamma^1}, \bar{\theta^1}, \dots, \bar{\beta}^k, \bar{\gamma}^k, \bar{\theta}^k, \bar{\alpha})$ and $\Tilde{p} \coloneqq \Tilde{\beta^1}, \Tilde{\gamma^1}, \Tilde{\theta^1}, \dots, \Tilde{\beta}^k, \Tilde{\gamma}^k, \Tilde{\theta}^k, \Tilde{\alpha})$ belongs to $\mathbb{P}$, and $\lambda_1, \lambda_2 > 0$ be such that
$$\lambda_1 \bar{p} + \lambda_2 \Tilde{p} = p,$$
and $\text{supp}(\bar{\alpha}), \text{supp}(\Tilde{\alpha}) \not \subset \text{supp}(\alpha)$.
Since $p$ has negative objective cost \eqref{objective}, without loss of generality, assume that $\bar{p}$ also has negative objective cost.
Since $\text{supp}(\bar{\alpha}) \not \subset \text{supp}(\alpha)$, there exists $\bar{j} \in \text{supp}(\bar{\alpha})$ but $\bar{j} \notin \alpha$.
Because $\beta, \gamma, \theta > 0$, if $\beta^i_j = \gamma^i_j = 0$ and $\theta^i_1 = \theta^i_2 = 0$, then $\bar{\beta}^i_j = \bar{\gamma}^i_j = 0$ and $\bar{\theta}^i_1 = \bar{\theta}^i_2 = 0$.
Moreover, by assumption, we have that $\alpha \neq \bar{w}$, so there exists $\bar{i}$ such that $\theta^{\bar{i}} = 0$.
This means that
$$\bar{\beta}^{\bar{i}}_{\bar{j}} - \bar{\gamma}^{\bar{i}}_{\bar{j}} + \bar{\alpha}_{\bar{j}} = \bar{\alpha}_{\bar{j}} \neq 0,$$
which contradicts the fact that $\bar{p} \in \mathbb{P}$.
Hence, $p$ is an extreme ray of $\mathbb{P}$.
\end{proof}

\section{Total Unimodularity of $\hat{A}$}
\label{appendix: totally unimodular}
\begin{theorem}
\label{tu}
The matrix $\hat{A}$ defined in \eqref{polyhedron} is totally unimodular.
\end{theorem}
\begin{proof}
We will prove that $\hat{A}$ is totally unimodular by constructing an equitable coloring \citep{ghouila1962caracterisation} for every column sub-matrix of $\hat{A}$. \\

\noindent
In order to do so, we show that we can color columns of any column sub-matrix of $A$ by blue and red such that the sum of blue columns subtracting the red ones results in a vector having the same sign as $\bar{w}$ component-wise.
Notionally, for a set $\bar{J} = \{\bar{j}_1, \bar{j}_2, \dots, \bar{j}_{|\bar{J}|} \} \subseteq \llbracket 2n + 2 \rrbracket$, we denote $A^{\bar{J}} = [A^{\bar{j}_1}, A^{\bar{j}_2}, \dots, A^{\bar{j}_{|J|}}]$ as a column sub-matrix of $A$.
In addition, we denote $\bar{B}, \bar{R} \in \bar{J}$, where $\bar{B} \cap \bar{R} = \emptyset$ and $\bar{B} \cup \bar{R} = \bar{J}$, be the blue and red columns respectively.
Furthermore, we define $w^{\bar{B}, \bar{R}} \coloneqq \sum_{\bar{j} \in \bar{B}}A^{\bar{j}} - \sum_{\bar{j} \in \bar{R}}A^{\bar{j}}$. 
We will prove that for every $\bar{J} \subseteq \llbracket 2n+2 \rrbracket$, there exists $\bar{B}, \bar{R}$ such that
\begin{equation}
\label{property *}
\tag{*}
    \begin{split}
        w^{\bar{B}, \bar{R}}_j \in \{0, 1\} & \text{ if } \bar{w}_j \in {0, 1} \ \forall j \in \llbracket n \rrbracket \\
        w^{\bar{B}, \bar{R}}_j \in \{0, -1\} & \text{ if } \bar{w}_j = -1 \ \forall j \in \llbracket n \rrbracket.
    \end{split}
\end{equation}
We observe that every column of $A$ can only be either $\pm \textbf{e}^j$ or $\pm \bar{w}$. 
If there is two column of $A^{\bar{J}}$ are opposite of each other, i.e., $j, n+j \in \bar{J}$ for some $j \in \llbracket n \llbracket$ or $2n+1, 2n+2 \in \bar{J}$, then by coloring both of them blue, their sum will be $\textbf{0}$.
Thus, we are left with the case where there does not exist any opposite pair of columns.
To this end, we consider the three following cases
\begin{enumerate}
    \item In case $A^{\bar{J}}$ does not contain $\bar{w}$ nor $-\bar{w}$: For every $\bar{j} \in \bar{J}$ where $A^{\bar{j}} = \textbf{e}^j$ for a $j \in \llbracket n \rrbracket$, we color it blue if $\bar{w}_j \in \{0, 1\}$, otherwise we color it red. For every $\bar{j} \in \bar{J}$ where $A^{\bar{j}} = -\textbf{e}^j$ for a $j \in \llbracket n \rrbracket$, we color
    it blue if $\bar{w}_j = -1$ and red otherwise. 
    \item In case $A^{\bar{J}}$ contains $\bar{w}$: We first color $\bar{w}$ blue. For every $\bar{j} \in \bar{J}$ where $A^{\bar{j}} = \textbf{e}^j$ for a $j \in \llbracket n \rrbracket$, we color it blue if $\bar{w}_j \in \{0, -1\}$, otherwise we color it red. For every $\bar{j} \in \bar{J}$ where $A^{\bar{j}} = -\textbf{e}^j$ for a $j \in \llbracket n \rrbracket$, we color it blue if $\bar{w}_j = 1$, and red otherwise.
    \item In case $A^{\bar{J}}$ contains $-\bar{w}$: We first color $-\bar{w}$ red, and color the remaining columns as in the case $A^{\bar{J}}$ contains $\bar{w}$.
\end{enumerate}
Apparently, by this way of coloring, the vector $w^{\bar{B}, \bar{R}}$ satisfies \eqref{property *}. 
Given such a coloring of column sub-matrix of $A$, we construct an equitable coloring for an arbitrary column sub-matrix of $\hat{A}$, which is given in the following general form:
\[
\breve{A} = 
\begin{bmatrix}
A_1 & 0 & \dots & 0 & I^J \\
0 & A_2 & \dots & 0 & I^J \\
\vdots & \vdots & \ddots & \vdots & \vdots \\
0 & 0 & \dots & A_k & I^J \\
\end{bmatrix},
\]
where $A_i$ are column sub-matrices of $A$ for every $i \in \llbracket k \rrbracket$, $J \subseteq \llbracket n \rrbracket$ and $I^J$ is a column sub-matrix of $I_n$.
We color columns that contain $A_1, \dots, A_k$ so that the coloring satisfies \eqref{property *}.
If $J$ is empty, then we immediately have an equitable coloring of $\breve{A}$.
On the other hand, if $J$ is non-empty, then for every $j \in \llbracket n \rrbracket$, we color $I^j$ blue if $\bar{w}_j \leq 0$ and red otherwise.
By \eqref{property *}, the difference of sum of blue columns and sum of red columns is a $\{0, \pm\}$-vector.
Hence, $\breve{A}$ has an equitable coloring.
Therefore, $\hat{A}$ is totally unimodular.
\end{proof}

\section{Experimental Details}
\label{app:exp_detail}
This section thoroughly describes every method we use to verify whether an image is robust given a threshold of perturbation $\epsilon > 0$.
For a given neural network $M$ and an image input $X_0$ with true label $l \in \llbracket N \rrbracket$ where $N$ is the number of possible output (classes) of the network, we want to know if there exists a perturbation $p$ satisfying $\|p\|_{\infty} \leq \epsilon$ such that $M(X_0 + p)$ output an label $l' \neq l$.
In a MIP approach, it is equivalent to solving the following problem:
\begin{equation}
\label{eq:target_atk}
    \begin{split}
        \max & \ c \cdot \mathcal{N}(X_0 + p) \\
        & \| p \|_{\infty} \leq \epsilon,
    \end{split}
\end{equation}
where $c^{l, l'}$ is a vector in $\mathds{R}^N$ such that $c^{l, l'}_l = -1$, $c^{l, l'}_{l'} = 1$, and $c^{l, l'}_j = 0$ for $j \notin \{l, l'\}$.
If the optimal value of the target attack problem \eqref{eq:target_atk} is positive, then exist a perturbation $p$ such that $\mathcal{N}(X_0 + p)$ outputs a different label than $l$, likely to be $l'$.
If for all $l' \in \llbracket N \rrbracket \setminus \{l\}$, the target attack problem with $c^{l, l'}$ has non-positive objective cost, then the image $X_0$ is invulnerable against perturbation within the $\epsilon$ threshold.\\

\noindent
In section \ref{sec:experiments}, we compare the Cayley embedding formulation with an Big-M formulation and our modified Deeppoly.
In the following, we formally describe the Big-M formulation and explain how we modify the Deeppoly method to work on any quantized activation function. 
\subsection{Big-M Formulation}
In section \ref{problem formulation}, we provide a general framework of how we turn a verification task into a MIP problem.
The key idea here is to model activation functions using additional binary variables and linear constraints.
Throughout the paper, whenever we mention Cayley embedding formulation, it means we are using Cayley embedding to model the graph of the activation functions.
Similarly, we say Big-M formulation to refer to a different MIP formulation for the activation functions.
Mathematically, let $f$ be a piecewise linear function defined as in \eqref{eq:uni_pwl}, $x, y$ be input and output of a neuron whose weight and bias are denoted as $w$ and $b$ respectively, the projection onto $(x, y)$ of the following MIP formulation gives the graph of $f(w \cdot x + b)$:
\begin{equation}
    \label{eq:Big-M}
    \begin{split}
        \sum_{i = 1}^k h_{i-1} z_i & \leq w \cdot x + b \leq \sum_{i = 1}^k h_i z_i \\
        M_1 (1 - z_1) & \leq y - a_1(w \cdot x + b) - b_1 \leq M_2 (1 - z_1)\\
        & \vdots \\
        M_1 (1 - z_k) & \leq y - a_k(w \cdot x + b) - b_k \leq M_2 (1 - z_k) \\
        \sum_{i = 1}^k z_i = & 1, \ z \in \{0, 1\}^k,
    \end{split}
\end{equation}
where $M_1, M_2$ are some lower bound and upper bound of $y$ respectively.
In our experiments, when the function $f$ is constant piecewise linear function, we in fact do not need the two value $M_1$ and $M_2$, because we can simply remove constraints that involve $M_1, M_2$ and add $y = \sum_{i = 1}^k b_i z_i$ instead.
Hence, the strength of our Big-M formulation only depends on the values $h_0$ and $h_k$.
These values are an estimated lower bound and upper bound of the pre-activation value $w \cdot x + b$, which are tightened using Deeppoly.
It is also worth mentioning that the same pre-activation bounds are also used for our Cayley embedding formulation.
\subsection{Quantized Deeppoly}
Originally, the Deeppoly method \citep{singh2019abstract} was not developed for quantized neural networks.
However, the key idea is simple enough to extend this class of networks.
For each neuron, we derive two linear functions, which serve as the lower bound and upper bound of the activation.
The two linear functions will be selected such that the volume of the region bounded by them is smallest.
Formally, given a neuron with non-decreasing piecewise constant activation $f$ with $h_0, h_1, \dots, h_k$ is its breakpoints.
As mentioned previously, the value $h_0$ and $h_k$ are pre-activation bound.
For the first hidden layer, $h_0$ and $h_k$ are computed by a simple interval arithmetic.
We assume that $h_{i+1} = h_i + h$ for every $i \in \llbracket k-2 \rrbracket$ and say $h$ is the step size of the breakpoints.
Furthermore, we suppose that $f$ can take values from $f_1 < f_2 < \dots < f_k$ and $f_{i+1} = f_i + f$ for every $i \in \llbracket k-1 \rrbracket$ where $f > 0$.
These assumptions apply to the Dorefa functions.
Under these conditions, the linear lower bound and upper bound for $f$ can be derived by the following rules:
\begin{enumerate}
    \item Let $u = h_k - h_{k-1}$ and $l = h_1 - h_0$.
    \item If $k = 2$, then
    \begin{itemize}
        \item If $u > l$, the coefficient for the upper bound linear function is $0$ and its constant term is $b_u = f_2$. The coefficient for the lower bound linear function is $c_l = \frac{f_2 - f_1}{u}$ and its constant term is $b_l = f_1 - (f_2 - f_1) \frac{h_1}{u}$.
        \item Otherwise, $c_u = \frac{f_2 - f_1}{l}$, $b_u = f_2 - (f_2 - f_1)\frac{h_1}{l}$, and $c_l = 0$, and $b_l = f_1$.
    \end{itemize}
    \item If $k > 2$, then
    \begin{itemize}
        \item If $l > h$, then $c_u = \frac{f_k - f_1}{h_{k-1} - h_0}$ and $b_u = f_k - (f_k - f_1)\frac{h_{k-1}}{h_{k-1} - h_0}$. Otherwise, $c_u = \frac{h}{f}$ and $b_u = f_k - h*h_{k-1}/f$.
        \item If $u > h$, then $c_l = \frac{f_k - f_1}{h_k - h_1}$ and $b_l = f_1 - (f_k - f_1)\frac{h_1}{h_k - h_1}$. Otherwise, $c_l = \frac{h}{f}$ and $b_l = f_1 - h*h_{k-1}/f$.
    \end{itemize}
\end{enumerate}

\begin{figure}[H]
\centering
\includegraphics[width=0.6\textwidth]{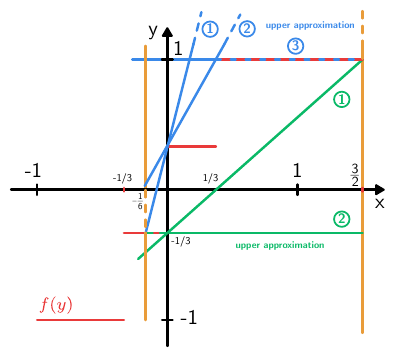}
\caption{\small An example for choosing the lower and upper approximation of activation function. Suppose the quantized activation function $f(x)$ is given as above (in red) and the input $x$ ranges from $\frac{-1}{6}$ to $\frac{3}{2}$. There are $3$ possible linear upper bound and $2$ possible linear lower bound in this case. And the ''best" bounds are third upper and the first lower approximations.}
\label{fig:deeppoly}
\end{figure}
\end{APPENDICES}

\theendnotes 

For the missing proofs, details on the fast separation procedure, the Big-M formulation that we used and the quantized DeepPoly, we refer to the electronic version of this paper.

\ACKNOWLEDGMENT{This work is supported by the Office of Naval Research award N00014-21-1-2262.}





\bibliography{ref}

\begin{thebibliography}{37}
\providecommand{\natexlab}[1]{#1}
\providecommand{\url}[1]{\texttt{#1}}
\expandafter\ifx\csname urlstyle\endcsname\relax
  \providecommand{\doi}[1]{doi: #1}\else
  \providecommand{\doi}{doi: \begingroup \urlstyle{rm}\Url}\fi

\bibitem[Agostinelli et~al.(2014)Agostinelli, Hoffman, Sadowski, and
  Baldi]{agostinelli2014learning}
Forest Agostinelli, Matthew Hoffman, Peter Sadowski, and Pierre Baldi.
\newblock Learning activation functions to improve deep neural networks.
\newblock \emph{arXiv preprint arXiv:1412.6830}, 2014.

\bibitem[Anderson et~al.(2020{\natexlab{a}})Anderson, Ma, Li, and
  Sojoudi]{anderson2020tightened}
Brendon~G Anderson, Ziye Ma, Jingqi Li, and Somayeh Sojoudi.
\newblock Tightened convex relaxations for neural network robustness
  certification.
\newblock In \emph{2020 59th {IEEE} {C}onference on {D}ecision and {C}ontrol
  (CDC)}, pages 2190--2197. {IEEE}, 2020{\natexlab{a}}.

\bibitem[Anderson et~al.(2020{\natexlab{b}})Anderson, Huchette, Ma,
  Tjandraatmadja, and Vielma]{anderson2020strong}
Ross Anderson, Joey Huchette, Will Ma, Christian Tjandraatmadja, and Juan~Pablo
  Vielma.
\newblock Strong mixed-integer programming formulations for trained neural
  networks.
\newblock \emph{Mathematical Programming}, 183\penalty0 (1):\penalty0 3--39,
  2020{\natexlab{b}}.

\bibitem[Bertsimas and Tsitsiklis(1997)]{bertsimas1997introduction}
Dimitris Bertsimas and John~N Tsitsiklis.
\newblock \emph{Introduction to linear optimization}, volume~6.
\newblock Athena Scientific Belmont, MA, 1997.

\bibitem[Bunel et~al.(2020)Bunel, Mudigonda, Turkaslan, Torr, Lu, and
  Kohli]{bunel2020branch}
Rudy Bunel, P~Mudigonda, Ilker Turkaslan, P~Torr, Jingyue Lu, and Pushmeet
  Kohli.
\newblock Branch and bound for piecewise linear neural network verification.
\newblock \emph{Journal of Machine Learning Research}, 21\penalty0 (2020),
  2020.

\bibitem[Carlini and Wagner(2017)]{carlini2017towards}
Nicholas Carlini and David Wagner.
\newblock Towards evaluating the robustness of neural networks.
\newblock In \emph{2017 {IEEE} {S}ymposium on {S}ecurity and {P}rivacy}, pages
  39--57. {IEEE}, 2017.

\bibitem[Conforti et~al.(2014)Conforti, Cornu{\'e}jols, Zambelli,
  et~al.]{conforti2014integer}
Michele Conforti, G{\'e}rard Cornu{\'e}jols, Giacomo Zambelli, et~al.
\newblock \emph{Integer {P}rogramming}, volume 271.
\newblock Springer, 2014.

\bibitem[Deng(2012)]{deng2012mnist}
Li~Deng.
\newblock The {MNIST} database of handwritten digit images for machine learning
  research.
\newblock \emph{{IEEE} {S}ignal {P}rocessing {M}agazine}, 29\penalty0
  (6):\penalty0 141--142, 2012.

\bibitem[Dvijotham et~al.(2018)Dvijotham, Stanforth, Gowal, Mann, and
  Kohli]{dvijotham2018dual}
Krishnamurthy Dvijotham, Robert Stanforth, Sven Gowal, Timothy~A Mann, and
  Pushmeet Kohli.
\newblock A dual approach to scalable verification of deep networks.
\newblock In \emph{UAI}, volume~1, page~3, 2018.

\bibitem[Eykholt et~al.(2018)Eykholt, Evtimov, Fernandes, Li, Rahmati, Xiao,
  Prakash, Kohno, and Song]{eykholt2018robust}
Kevin Eykholt, Ivan Evtimov, Earlence Fernandes, Bo~Li, Amir Rahmati, Chaowei
  Xiao, Atul Prakash, Tadayoshi Kohno, and Dawn Song.
\newblock Robust physical-world attacks on deep learning visual classification.
\newblock In \emph{Proceedings of the {IEEE} {C}onference on {C}omputer
  {V}ision and {P}attern {R}ecognition}, pages 1625--1634, 2018.

\bibitem[Geiger and Team(2020)]{larq}
Lukas Geiger and Plumerai Team.
\newblock Larq: An open-source library for training binarized neural networks.
\newblock \emph{Journal of Open Source Software}, 5\penalty0 (45):\penalty0
  1746, January 2020.
\newblock \doi{10.21105/joss.01746}.
\newblock URL \url{https://doi.org/10.21105/joss.01746}.

\bibitem[Ghouila-Houri(1962)]{ghouila1962caracterisation}
Alain Ghouila-Houri.
\newblock Caract{\'e}risation des matrices totalement unimodulaires.
\newblock \emph{Comptes Redus Hebdomadaires des S{\'e}ances de l'Acad{\'e}mie
  des Sciences (Paris)}, 254:\penalty0 1192--1194, 1962.

\bibitem[Goodfellow et~al.(2014)Goodfellow, Shlens, and
  Szegedy]{goodfellow2014explaining}
Ian~J Goodfellow, Jonathon Shlens, and Christian Szegedy.
\newblock Explaining and harnessing adversarial examples.
\newblock \emph{arXiv preprint arXiv:1412.6572}, 2014.

\bibitem[Grosse et~al.(2016)Grosse, Papernot, Manoharan, Backes, and
  McDaniel]{grosse2016adversarial}
Kathrin Grosse, Nicolas Papernot, Praveen Manoharan, Michael Backes, and
  Patrick McDaniel.
\newblock Adversarial perturbations against deep neural networks for malware
  classification.
\newblock \emph{arXiv preprint arXiv:1606.04435}, 2016.

\bibitem[{Gurobi Optimization, LLC}(2022)]{gurobi}
{Gurobi Optimization, LLC}.
\newblock {Gurobi Optimizer Reference Manual}, 2022.
\newblock URL \url{https://www.gurobi.com}.

\bibitem[Han and G{\'o}mez(2021)]{han2021single}
Shaoning Han and Andr{\'e}s G{\'o}mez.
\newblock Single-neuron convexification for binarized neural networks, 2021.
\newblock URL \url{https://optimization-online.org/?p=17148}.

\bibitem[Hubara et~al.(2016)Hubara, Courbariaux, Soudry, El-Yaniv, and
  Bengio]{hubara2016binarized}
Itay Hubara, Matthieu Courbariaux, Daniel Soudry, Ran El-Yaniv, and Yoshua
  Bengio.
\newblock Binarized neural networks.
\newblock \emph{{A}dvances in {N}eural {I}nformation {P}rocessing {S}ystems},
  29, 2016.

\bibitem[Hubara et~al.(2017)Hubara, Courbariaux, Soudry, El-Yaniv, and
  Bengio]{hubara2017quantized}
Itay Hubara, Matthieu Courbariaux, Daniel Soudry, Ran El-Yaniv, and Yoshua
  Bengio.
\newblock Quantized neural networks: Training neural networks with low
  precision weights and activations.
\newblock \emph{The Journal of Machine Learning Research}, 18\penalty0
  (1):\penalty0 6869--6898, 2017.

\bibitem[Huber et~al.(2000)Huber, Rambau, and
  G{\'o}mez-P{\'e}rez]{huber2000cayley}
Birkett Huber, J{\"o}rg Rambau, and Domingo G{\'o}mez-P{\'e}rez.
\newblock The {C}ayley trick, lifting subdivisions and the {B}ohne-{D}ress
  theorem on zonotopal tilings.
\newblock \emph{Journal of the European Mathematical Society}, 2\penalty0
  (2):\penalty0 179--198, 2000.

\bibitem[Jia and Liang(2017)]{jia2017adversarial}
Robin Jia and Percy Liang.
\newblock Adversarial examples for evaluating reading comprehension systems.
\newblock \emph{arXiv preprint arXiv:1707.07328}, 2017.

\bibitem[Katz et~al.(2017)Katz, Barrett, Dill, Julian, and
  Kochenderfer]{katz2017reluplex}
Guy Katz, Clark Barrett, David~L Dill, Kyle Julian, and Mykel~J Kochenderfer.
\newblock Reluplex: An efficient {SMT} solver for verifying deep neural
  networks.
\newblock In \emph{International {C}onference on {C}omputer {A}ided
  {V}erification}, pages 97--117. Springer, 2017.

\bibitem[LeCun et~al.(1995)LeCun, Bengio, et~al.]{lecun1995convolutional}
Yann LeCun, Yoshua Bengio, et~al.
\newblock Convolutional networks for images, speech, and time series.
\newblock \emph{The {H}andbook of {B}rain {T}heory and {N}eural {N}etworks},
  3361\penalty0 (10):\penalty0 1995, 1995.

\bibitem[Lin et~al.(2017)Lin, Zhao, and Pan]{lin2017towards}
Xiaofan Lin, Cong Zhao, and Wei Pan.
\newblock Towards accurate binary convolutional neural network.
\newblock \emph{arXiv preprint arXiv:1711.11294}, 2017.

\bibitem[Papernot et~al.(2016)Papernot, McDaniel, Jha, Fredrikson, Celik, and
  Swami]{papernot2016limitations}
Nicolas Papernot, Patrick McDaniel, Somesh Jha, Matt Fredrikson, Z~Berkay
  Celik, and Ananthram Swami.
\newblock The limitations of deep learning in adversarial settings.
\newblock In \emph{2016 {IEEE} European {S}ymposium on {S}ecurity and {P}rivacy
  (EuroS\&P)}, pages 372--387. {IEEE}, 2016.

\bibitem[Rastegari et~al.(2016)Rastegari, Ordonez, Redmon, and
  Farhadi]{rastegari2016xnor}
Mohammad Rastegari, Vicente Ordonez, Joseph Redmon, and Ali Farhadi.
\newblock Xnor-net: Imagenet classification using binary convolutional neural
  networks.
\newblock In \emph{European {C}onference on {C}omputer {V}ision}, pages
  525--542. Springer, 2016.

\bibitem[Schmidhuber(2015)]{schmidhuber2015deep}
J{\"u}rgen Schmidhuber.
\newblock Deep learning in neural networks: An overview.
\newblock \emph{{N}eural {N}etworks}, 61:\penalty0 85--117, 2015.

\bibitem[Sharif et~al.(2016)Sharif, Bhagavatula, Bauer, and
  Reiter]{sharif2016accessorize}
Mahmood Sharif, Sruti Bhagavatula, Lujo Bauer, and Michael~K Reiter.
\newblock Accessorize to a crime: Real and stealthy attacks on state-of-the-art
  face recognition.
\newblock In \emph{Proceedings of the 2016 {ACM} {SIGSAC} {C}onference on
  {C}omputer and {C}ommunications {S}ecurity}, pages 1528--1540, 2016.

\bibitem[Singh et~al.(2019)Singh, Gehr, P{\"u}schel, and
  Vechev]{singh2019abstract}
Gagandeep Singh, Timon Gehr, Markus P{\"u}schel, and Martin Vechev.
\newblock An abstract domain for certifying neural networks.
\newblock \emph{{P}roceedings of the {ACM} on {P}rogramming {L}anguages},
  3\penalty0 (POPL):\penalty0 1--30, 2019.

\bibitem[Szegedy et~al.(2014)Szegedy, Zaremba, Sutskever, Bruna, Erhan,
  Goodfellow, and Fergus]{szegedy2014intriguing}
Christian Szegedy, Wojciech Zaremba, Ilya Sutskever, Joan Bruna, Dumitru Erhan,
  Ian Goodfellow, and Rob Fergus.
\newblock Intriguing properties of neural networks.
\newblock In \emph{2nd {I}nternational {C}onference on {L}earning
  {R}epresentations, {ICLR} 2014}, 2014.

\bibitem[Tang et~al.(2017)Tang, Hua, and Wang]{tang2017train}
Wei Tang, Gang Hua, and Liang Wang.
\newblock How to train a compact binary neural network with high accuracy?
\newblock In \emph{Thirty-First {AAAI} {C}onference on {A}rtificial
  {I}ntelligence}, 2017.

\bibitem[Tjandraatmadja et~al.(2020)Tjandraatmadja, Anderson, Huchette, Ma,
  Patel, and Vielma]{tjandraatmadja2020convex}
Christian Tjandraatmadja, Ross Anderson, Joey Huchette, Will Ma, Krunal~Kishor
  Patel, and Juan~Pablo Vielma.
\newblock The convex relaxation barrier, revisited: Tightened single-neuron
  relaxations for neural network verification.
\newblock \emph{{A}dvances in {N}eural {I}nformation {P}rocessing {S}ystems},
  33, 2020.

\bibitem[Vielma(2015)]{vielma2015mixed}
Juan~Pablo Vielma.
\newblock Mixed integer linear programming formulation techniques.
\newblock \emph{Siam Review}, 57\penalty0 (1):\penalty0 3--57, 2015.

\bibitem[Vielma(2018)]{vielma2018embedding}
Juan~Pablo Vielma.
\newblock Embedding formulations and complexity for unions of polyhedra.
\newblock \emph{Management Science}, 64\penalty0 (10):\penalty0 4721--4734,
  2018.

\bibitem[Vielma(2019)]{vielma2019small}
Juan~Pablo Vielma.
\newblock Small and strong formulations for unions of convex sets from the
  {C}ayley embedding.
\newblock \emph{{M}athematical {P}rogramming}, 177\penalty0 (1):\penalty0
  21--53, 2019.

\bibitem[Xu et~al.(2020)Xu, Ma, Liu, Deb, Liu, Tang, and
  Jain]{xu2020adversarial}
Han Xu, Yao Ma, Hao-Chen Liu, Debayan Deb, Hui Liu, Ji-Liang Tang, and Anil~K
  Jain.
\newblock Adversarial attacks and defenses in images, graphs and text: A
  review.
\newblock \emph{{I}nternational {J}ournal of {A}utomation and {C}omputing},
  17\penalty0 (2):\penalty0 151--178, 2020.

\bibitem[Zeng et~al.(2010)Zeng, Huang, and Zheng]{zeng2010multistability}
Zhigang Zeng, Tingwen Huang, and Wei~Xing Zheng.
\newblock Multistability of recurrent neural networks with time-varying delays
  and the piecewise linear activation function.
\newblock \emph{{IEEE} {T}ransactions on {N}eural {N}etworks}, 21\penalty0
  (8):\penalty0 1371--1377, 2010.

\bibitem[Zhou et~al.(2016)Zhou, Wu, Ni, Zhou, Wen, and Zou]{zhou2016dorefa}
Shuchang Zhou, Yuxin Wu, Zekun Ni, Xinyu Zhou, He~Wen, and Yuheng Zou.
\newblock Dorefa-net: Training low bitwidth convolutional neural networks with
  low bitwidth gradients.
\newblock \emph{arXiv preprint arXiv:1606.06160}, 2016.

\end{thebibliography}
\bibliographystyle{plainnat}

\end{document}